\documentclass{article} 
\usepackage{iclr2027_conference,times}
\newif\ifpreprint
\preprinttrue
\ifpreprint\iclrfinalcopy\fi

\usepackage{amsmath,amsfonts,bm}

\def\eqref#1{equation~\ref{#1}}

\def\1{\bm{1}}

\DeclareMathAlphabet{\mathsfit}{\encodingdefault}{\sfdefault}{m}{sl}
\SetMathAlphabet{\mathsfit}{bold}{\encodingdefault}{\sfdefault}{bx}{n}

\usepackage[utf8]{inputenc} 
\usepackage[T1]{fontenc}    
\usepackage{hyperref}       
\usepackage{url}            
\usepackage{booktabs}       
\usepackage{amsfonts}       
\usepackage{nicefrac}       
\usepackage{microtype}      
\usepackage[table]{xcolor}

\usepackage{graphicx}
\usepackage{amsmath}
\usepackage{amssymb}

\usepackage{listings}
\usepackage{mwe} 
\usepackage{makecell}
\usepackage{color, colortbl}
\usepackage[normalem]{ulem} 
\usepackage{algorithm}
\usepackage{algpseudocode}
\algrenewcomment[1]{\hfill$\triangleright$~#1}
\algrenewcommand\algorithmicrequire{\textbf{Input:}}
\algrenewcommand\algorithmicensure{\textbf{Return:}}

\usepackage{wrapfig}
\usepackage{caption}
\usepackage{pifont} 
\usepackage{multirow}
\usepackage{booktabs}
\usepackage{float}

\usepackage{chngpage}
\usepackage{wrapfig}
\usepackage{adjustbox}
\usepackage{comment}
\usepackage{color}

\usepackage{listings}

\usepackage{subcaption}

\usepackage{bm}
\usepackage[capitalize]{cleveref}

\usepackage[toc]{appendix}
\usepackage{etoc}
\usepackage{minitoc}
\usepackage{mathtools}

\usepackage{amsthm}

\usepackage{titletoc}
\usepackage{enumitem}

\newtheorem{theorem}{Theorem}[section]        
\newtheorem{lemma}[theorem]{Lemma}            
\newtheorem{proposition}[theorem]{Proposition}
\newtheorem{corollary}[theorem]{Corollary}    

\theoremstyle{definition}

\theoremstyle{remark}
\newtheorem{remark}[theorem]{Remark}          

\makeatletter
\newcommand{\@makeappendixtitle}{%
  \vbox{%
    \hsize\textwidth
    \linewidth\hsize
    \vskip 0.1in
    \centering
    {\LARGE\bfseries Appendix of \@title\par}%
    \vskip 0.3in\@minus0.1in
  }%
}

\newcommand{\appendixtitle}[1][]{%
  \clearpage
  \thispagestyle{empty}%
  \centering {\huge\textbf{Appendix}}
  \let\appendixtitle\relax
}
\makeatother

\usepackage{booktabs, multirow, graphicx}
\usepackage[table]{xcolor}
\usepackage{colortbl}
\definecolor{ourshl}{HTML}{E5EEDF}
\definecolor{cnpgreen}{RGB}{43, 122, 63}

\usepackage{amsthm}
\usepackage[most]{tcolorbox}

\newcommand{\second}[1]{\underline{#1}}

\crefname{section}{Sec.}{Secs.}
\Crefname{section}{Section}{Sections}
\Crefname{table}{Table}{Tables}
\crefname{table}{Tab.}{Tabs.}
\etocsettocstyle{\section*{Appendix}}{}

\definecolor{Gray}{gray}{0.9}
\definecolor{ImportantColor}{rgb}{0.63, 0.79, 0.95}
\newcolumntype{g}{>{\columncolor{ImportantColor}}c}
\newcolumntype{?}{!{\vrule width 1pt}}
\usepackage[normalem]{ulem}

\newcommand{\paratitle}[1]{\par\noindent\textbf{#1}}

\graphicspath{
{figs/},
}

\newcommand*{\belowrulesepcolor}[1]{%
  \noalign{%
    \kern-\belowrulesep 
    \begingroup 
      \color{#1}%
      \hrule height\belowrulesep 
    \endgroup 
  }%
} 
\newcommand*{\aboverulesepcolor}[1]{%
  \noalign{%
    \begingroup 
      \color{#1}%
      \hrule height\aboverulesep 
    \endgroup 
    \kern-\aboverulesep 
  }%
}

\definecolor{codegreen}{rgb}{0,0.6,0}
\definecolor{codegray}{rgb}{0.5,0.5,0.5}
\definecolor{codepurple}{rgb}{0.58,0,0.82}
\definecolor{backcolour}{rgb}{0.95,0.95,0.92}

\lstdefinestyle{mystyle}{
    backgroundcolor=\color{backcolour},   
    commentstyle=\color{codegreen},
    keywordstyle=\color{magenta},
    numberstyle=\tiny\color{codegray},
    stringstyle=\color{codepurple},
    basicstyle=\ttfamily\footnotesize,
    breakatwhitespace=false,         
    breaklines=true,                 
    captionpos=b,                    
    keepspaces=true,                 
    numbers=left,                    
    numbersep=5pt,                  
    showspaces=false,                
    showstringspaces=false,
    showtabs=false,                  
    tabsize=2
}

\hypersetup{
    colorlinks=true,
    linkcolor=black,
}

\hypersetup{hidelinks}
\renewcommand{\bottomfraction}{0.8}
\renewcommand{\floatpagefraction}{0.8}

\newcommand{\restoreappendixlayout}{%
\setcounter{topnumber}{2}
\setcounter{totalnumber}{3}
\setlength{\textfloatsep}{20pt plus 2pt minus 4pt}
\setlength{\floatsep}{12pt plus 2pt minus 2pt}
\setlength{\intextsep}{12pt plus 2pt minus 2pt}
\renewcommand{\bottomfraction}{0.3}
\renewcommand{\floatpagefraction}{0.5}
\captionsetup{font=normalsize,skip=10pt}
}

\newcommand{\method}{CNP-Flow}
\newcommand{\methodfull}{\textbf{C}onditional \textbf{N}oise \textbf{P}redictor for \textbf{Flow} Matching}

\title{Learning Conditional Source Distribution via Flow Reversal for Temporal Flow Matching}

\author{%
Kuan-Hsun Tu$^{*}$, Hsuan-Chi Liu$^{*\dagger}$, Jia-Wei Liao$^{*}$, Chien-Sheng Chiang, Cheng-Fu Chou \& Tsung-Wei Ke\\
National Taiwan University\\
\texttt{\{khtu, hhhhhsuan, d11922016, b10902063, ccf, twke\}@csie.ntu.edu.tw}%
}

\begin{document}

\maketitle
\ifpreprint
\lhead{Preprint}
{\renewcommand{\thefootnote}{\fnsymbol{footnote}}\footnotetext[1]{Equal contribution. \quad $^{\dagger}$Corresponding author.}}
\fi

\begin{figure}[H]
    \centering
    \includegraphics[width=\linewidth]{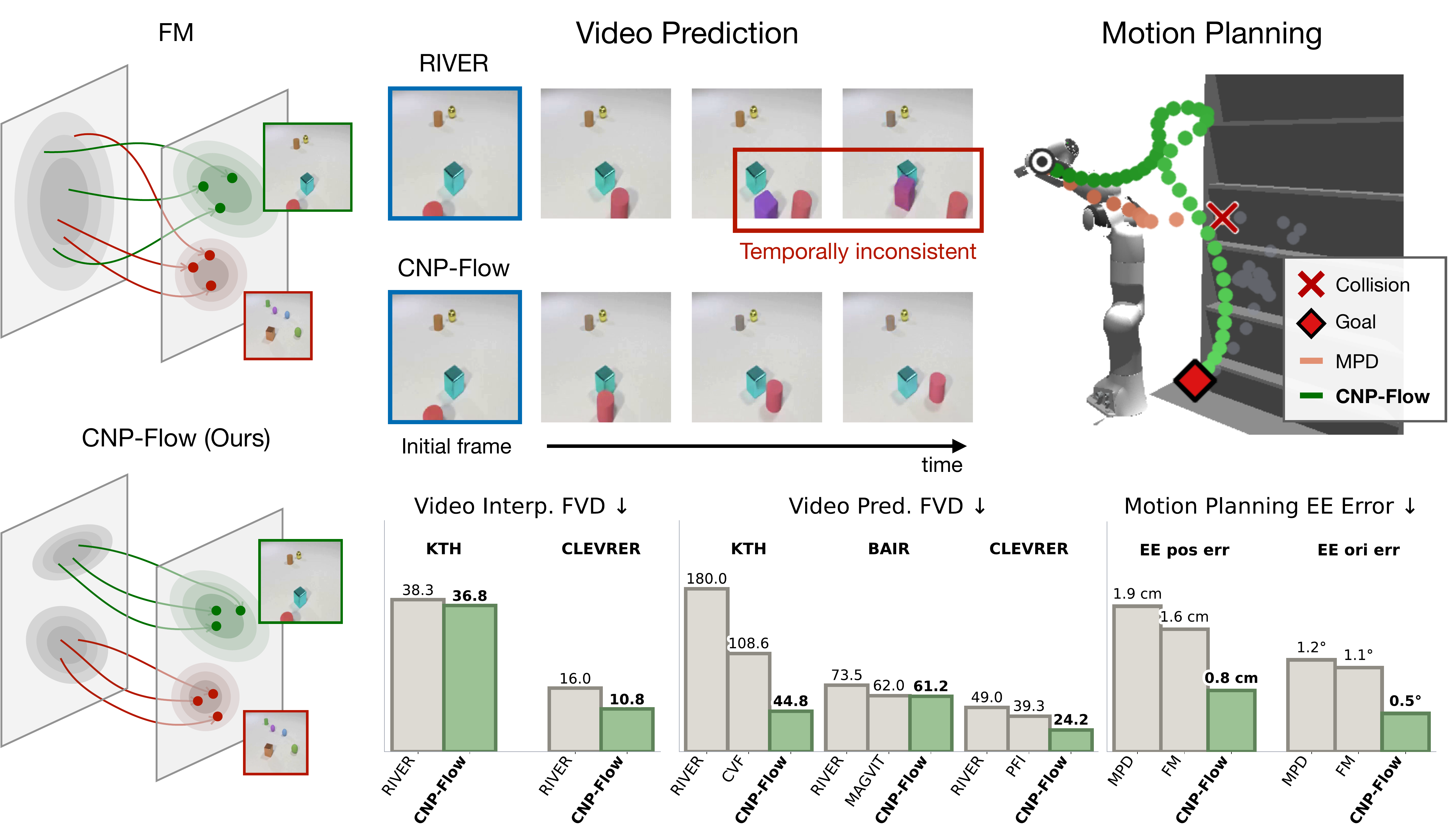}
    \caption{
    \textbf{CNP-Flow learns condition-dependent source distributions for temporal generation.}
    \method{} uses a Conditional Noise Predictor (CNP) to predict a source distribution for each condition. This source-side conditioning reduces FVD by up to $75\%$ over RIVER on video prediction, and yields more precise motion-planning trajectories with fewer collisions.}
    \label{fig:teaser}
\end{figure}
\vspace{-10px}
\begin{abstract}
We introduce CNP-Flow, a flow matching framework for temporal generation that learns conditional source distributions through flow reversal. Whereas standard conditional flow matching (FM) incorporates conditioning through the vector field and draws source samples from a standard Gaussian, CNP-Flow uses a conditional noise predictor (CNP) to produce an isotropic Gaussian source for each temporal condition. The CNP is supervised by source samples obtained through flow reversal, which maps observed targets backward through a pretrained FM model. A three-stage pipeline pretrains the FM model, trains the CNP, and fine-tunes the FM model using the learned source distribution, while preserving the FM backbone architecture. Across video prediction, video interpolation, and 7-DoF Franka robot motion planning, CNP-Flow consistently improves generation quality. It also matches baseline performance with fewer function evaluations. Project page: \href{https://embodiedai-ntu.github.io/cnpflow/index.html}{\textcolor{cnpgreen}{https://embodiedai-ntu.github.io/cnpflow}}
\end{abstract}


\section{Introduction}

Generative modeling is typically framed as drawing samples from a learned data distribution, where each generation is independent of the others. However, applications such as video generation~\citep{ho2022video, ge2023preserve} and robotic control~\citep{janner2022planning, carvalho2023motion} require a more challenging and constrained setup: they must generate a sequence of temporally correlated samples conditioned on past or future data. This defines the task of \textit{temporal generation}, which imposes two requirements: the transitions between samples must follow those of the data distribution, and each sample must remain structurally consistent with the conditional inputs~\citep{lai2018learning, ge2023preserve}. Depending on the conditioning, temporal generation can be further divided into extrapolation and interpolation. Extrapolation predicts future samples given the past~\citep{voleti2022mcvd, davtyan2023river}, while interpolation generates intermediate trajectories given the start and end of a sequence~\citep{janner2022planning, carvalho2023motion,reda2022film}.

Flow Matching (FM)~\citep{lipman2023flow} is an effective framework for generative modeling: a neural network regresses a deterministic vector field that transports samples from a source distribution to the target distribution. The source is typically chosen to be simple, such as a standard Gaussian. Thanks to this general formulation, FM models extend naturally to temporal generation by treating temporal conditions as additional inputs~\citep{esser2024scaling,3d_flowmatch_actor}. While conceptually simple, their performance depends heavily on the quantity and quality of the conditions injected into the model~\citep{peebles2023scalable}. Existing approaches, therefore, focus on improving the conditioning mechanism through various inductive biases: randomly masking conditional inputs as a form of regularization~\citep{voleti2022mcvd,davtyan2023river} or biasing generated samples toward conditional data via gradient methods~\citep{ho2022video}. However, the former sacrifices fidelity, while the latter cannot handle multiple conditions simultaneously.

These approaches all target the conditioning mechanism. A separate line of work, however, takes a different angle: modulating the source distribution itself. Recent studies show that FM models can be steered toward higher-quality outputs by selecting more suitable source samples~\citep{pooladian2023multisample, tong2024improving, chiang2026cot}. A common approach in image editing~\citep{meng2021sdedit, song2020score} replaces the standard Gaussian noise with a noised version of the original image, thereby preserving its spatial structure. In text-to-video generation~\citep{ge2023preserve} and video editing~\citep{chang2025warped}, drawing source samples in a directionally correlated manner across frames improves temporal coherence over drawing them independently per frame. Pushing this further, recent video prediction methods~\citep{shrivastava2024cvp,shrivastava2024cvf} replace the standard Gaussian source with the previous frame. In robotic manipulation, action generation is largely determined by the source samples, so training a dedicated noise sampler enables adaptation of pretrained robot policies~\citep{wagenmaker2025steering}. Together, these results suggest that source samples themselves carry rich conditional information, which motivates us to ask: can a learned source distribution also benefit temporal generation?

We introduce \methodfull{} (\method{}), a flow matching framework for temporal generation. The core idea is to learn a noise predictor that takes the current condition as input and outputs a source distribution suited to downstream FM models under that condition. This formulation naturally encompasses both interpolation and extrapolation (see Fig.~\ref{fig:teaser}). To supervise the noise predictor, we construct pseudo-ground-truth source distributions by retrieving source samples for each condition through flow reversal of a pretrained FM model. A recent study~\citep{chiang2026cot} shows that pretrained FM models map similar target data to nearby source samples through flow reversal. Motivated by this, we parameterize these distributions as a family of isotropic Gaussians, with the conditional noise predictor (CNP) outputting their mean and standard deviation. This sidesteps the exact search for optimal source distributions, which is intractable, while still providing effective surrogates that substantially improve temporal generation in our experiments.

\begin{figure}[t]
    \centering
    \includegraphics[width=\linewidth]{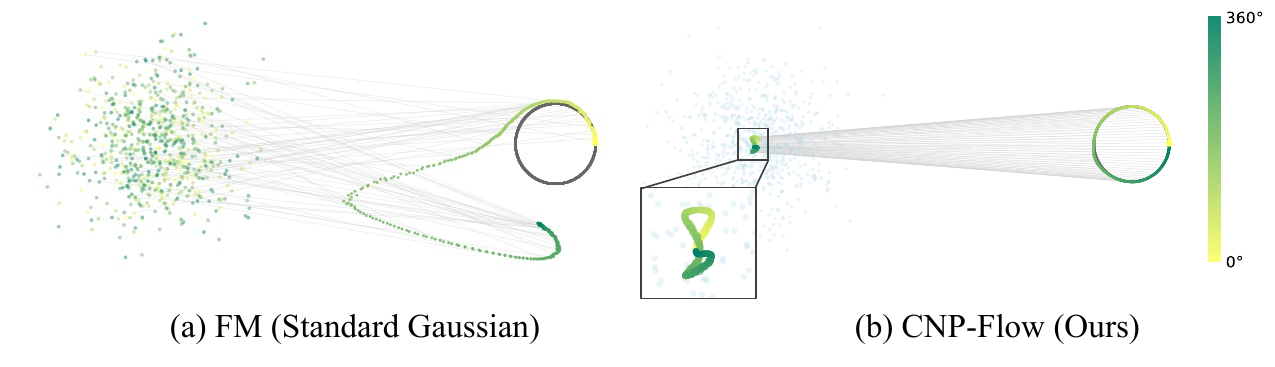}
    \caption{\textbf{2D toy example of temporal generation}. Colors indicate temporal phase along the target circular trajectory. (a) Vanilla FM draws source samples from a standard Gaussian. Samples associated with neighboring temporal conditions can be widely separated, and the generated sequence deviates from the target trajectory. (b) CNP-Flow learns a conditional source distribution with a smoother arrangement of samples by temporal phase (see the zoomed region), and the generated sequence follows the target trajectory more closely. Detailed experimental settings are provided in Appendix~\ref{app:2d_toy_example}.
    }
    \label{fig:toy_example}
\end{figure}

\method{} consists of a CNP and an FM model, trained via a three-stage pipeline that bootstraps each model from the other. First, we pretrain an FM model with a standard Gaussian source; flow reversal through the pretrained model then retrieves surrogate source samples for each condition. Next, we train the noise predictor by maximizing the likelihood of these source samples under each condition, enabling it to predict source distributions at test time. Finally, we fine-tune the FM model with the predicted distributions, refining its vector field. Fig.~\ref{fig:toy_example} illustrates this on a toy example: compared to a vanilla FM, \method{} draws temporally correlated source samples and produces temporally coherent sequences.

\method{} shows empirical gains along three independent axes. On video benchmarks, distribution-level alignment improves, with FVD on KTH dropping from $180.00$ to $44.76$ and comparable margins on CLEVRER and BAIR (see Fig.~\ref{fig:teaser}). Inference becomes far more efficient: with only $5$ function evaluations, \method{} surpasses what the standard Gaussian baseline achieves at $50$. The same formulation transfers beyond video to a 7-DoF Franka motion planning task, where the fraction of valid trajectories rises from $84.9\%$ to $88.5\%$, and end-effector position error drops from $1.6$ cm to $0.8$ cm (see Fig.~\ref{fig:teaser}). These results point to the same mechanism: conditioning the source improves source--target alignment and straightens the transport path, without modality-specific structure.

\section{Related Work}

\paratitle{Temporal conditions.} Generating temporally coherent samples is a central challenge in video prediction and motion planning. Early approaches based on recurrent networks~\citep{graves2012long, babaeizadeh2021fitvid} and autoregressive models~\citep{vaswani2017attention, le2021ccvs, ma2024latte} struggle with long-term temporal consistency or computational cost. Recently, diffusion~\citep{ho2020denoising, dhariwal2021diffusion, song2020score} and flow models~\citep{lipman2023flow}, driven by their success in high-fidelity image synthesis, have been extended to video generation, prediction~\citep{ho2022video, voleti2022mcvd, hoppe2022diffusion, yang2023diffusion, davtyan2023river} and motion planning~\citep{janner2022planning, carvalho2023motion, carvalho2025motion}, exploring the strategies to inject temporal conditions into the model. More specifically,~\citet{voleti2022mcvd} condition on all past and future frames but randomly mask them during training, enabling a single model to handle prediction, generation, and interpolation.~\citet{davtyan2023river} instead condition on a reference frame and one past frame to reduce training cost.~\citet{janner2022planning} condition on the first frame and last frame to generate a feasible trajectory. Alternatively, \citet{ho2022video} bias the sampling trajectory toward conditioning frames via gradient-based reconstruction guidance, but the per-condition gradient pass scales poorly when multiple temporal conditions must be satisfied at once. Despite their differences, these methods all focus on how the condition enters the generation process, while assuming the source distribution is a fixed standard Gaussian.

\paratitle{Optimized noise prior.} A separate line of work shapes the source distribution. In image generation, noise priors have been optimized through reward functions~\citep{eyring2024reno}, structured patterns~\citep{samuel2024generating, mao2024lottery}, pre-collected databases~\citep{mao2023semantic}, or condition-dependent shifts~\citep{chen2025car}. For video synthesis, \citet{ge2023preserve} introduce temporally correlated noise across frames, while \citet{luo2023videofusion} and \citet{zhang2024trip} model residual noise between consecutive frames. In video-to-video settings, \citet{chang2025warped, burgert2025go, liu2025equivdm} warp the noise to align with the target motion. For video prediction, \citet{shrivastava2024cvp} and \citet{chen2024pfi} bypass the standard Gaussian noise by modeling frame transitions as continuous interpolation. Despite their variety, these methods each impose a specific noise structure, such as temporal correlations, geometric warping, or frame-wise residuals, that either degrades over long horizons or sacrifices multimodality~\citep {kim2026better, lee2025there, sadat2023cads}. In robotic manipulation, \citet{wagenmaker2025steering} train a dedicated noise sampler to adapt a pretrained diffusion policy to new tasks, though their approach operates in low-dimensional action spaces and its scalability to high-dimensional settings, such as video, remains unclear. \citet{chiang2026cot} show that pretrained flow-matching models map similar targets to nearby source regions, but their focus is on image generation and optimal transport rather than temporal generation. To our knowledge, no prior work has learned a condition-dependent source distribution that generalizes across temporal generation tasks while preserving stochasticity.

\section{Preliminary}
Flow matching (FM) is a generative framework that learns to transport samples from a source distribution $p_0$ to a target distribution $p_1$. A neural network  with parameters $\theta$ is trained to regress a time-varying vector field $\bm{v}_\theta$, which defines an ordinary differential equation (ODE): $\mathrm{d}\bm{x}_t = \bm{v}_\theta(\bm{x}_t, t)\,\mathrm{d}t$. Integrating this ODE from $t=0$ to $t=1$ transports a source sample $\bm{x}_0$ to a target sample $\bm{x}_1$. FM extends naturally to conditional generation by conditioning both the vector field $\bm{v}_\theta(\bm{x}_t, t \mid \mathcal{C})$ and the target distribution $p_1(\cdot \mid \mathcal{C})$ on a conditioning input $\mathcal{C}$.

A standard training procedure pairs a source sample $\bm{x}_0$ with a target sample $\bm{x}_1$, samples a timestep $t \sim \mathcal{U}(0,1)$, and constructs an intermediate sample by linear interpolation: $\bm{x}_t = (1-t)\bm{x}_0 + t\bm{x}_1$. The model is then optimized with the following objective~\citep{lipman2023flow}:
\begin{equation}
\mathcal{L}_{\mathrm{FM}}(\theta) = \mathbb{E}_{\mathcal{C}, \bm{x}_0 \sim p_0, \bm{x}_1 \sim p_1(\cdot \mid \mathcal{C}), t \sim \mathcal{U}(0,1)} \left[\|\bm{v}_\theta(\bm{x}_t, t \mid \mathcal{C}) - (\bm{x}_1 - \bm{x}_0)\|_2^2 \right].
\end{equation}
At inference time, generation proceeds by integrating the learned vector field from a source sample:
\begin{equation}
\hat{\bm{x}}_1 := \bm{x}_0 + \int_0^1 \bm{v}_\theta(\bm{x}_t, t \mid \mathcal{C}) \, \mathrm{d}t.
\label{eqn:fm_inference}
\end{equation}
Eqn.~\ref{eqn:fm_inference} shows that the generated sample depends on both the condition $\mathcal{C}$ and the source sample $\bm{x}_0$. This observation motivates us to investigate whether more suitable source samples can improve generation quality under each condition.

\begin{figure}[t]
    \centering
    \includegraphics[width=\linewidth]{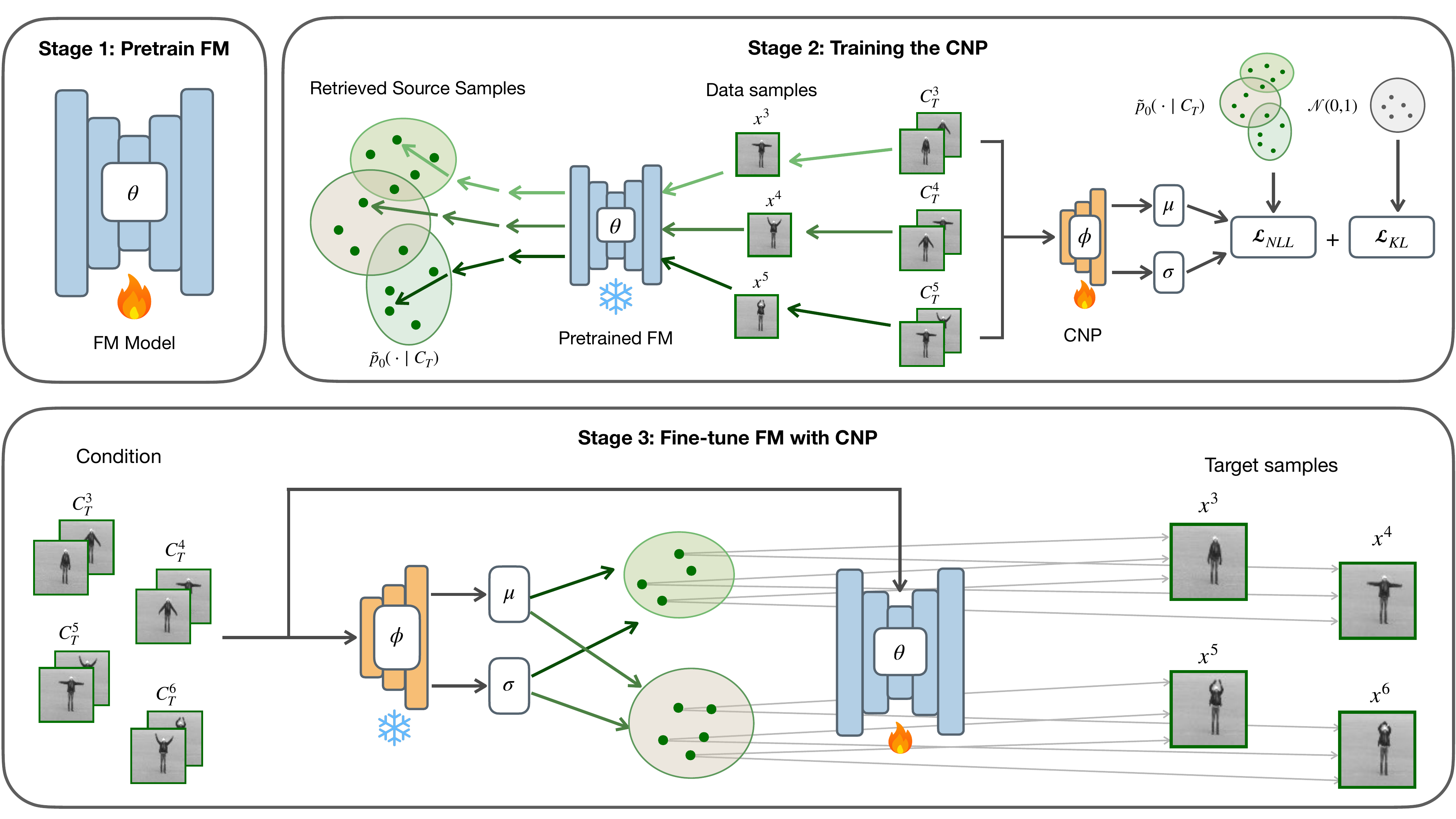}
    \caption{\textbf{Overview of the proposed three-stage training pipeline.} Stage 1: A pretrained FM is used to obtain flow-reversal-induced source samples, defining a source distribution aligned with the target data. Stage 2: The CNP is trained to model this distribution by predicting a conditional Gaussian prior via NLL and KL objectives. Stage 3: The FM model is fine-tuned using source samples drawn from the learned conditional prior, resulting in more coherent and target-aligned transport trajectories.}
    \label{fig:train}
\end{figure}

\section{Method}
\label{sec:method}
We propose \method{}, a temporal generation framework that augments FM with a learned conditional source prior. Standard FM models draw source samples from a standard Gaussian and take the condition as an additional input. However, they often fail to generate high-quality data that faithfully follows the condition. \method{} reinforces the conditioning mechanism by using a conditional noise predictor (CNP) that seeks a source distribution suitable for generating target data under each condition. Our design is general in two ways: (1) the noise predictor handles different types of conditioning, including both interpolation and extrapolation, allowing \method{} to unify these tasks in a single framework; and (2) \method{} only modulates the source distribution supplied to the FM model without changing its architecture, so it can be integrated into existing FM frameworks.

In the following, we describe the training and inference procedure of \method{} in Sec.~\ref{subsec:train_test}, and its application to downstream temporal generation tasks in Sec.~\ref{subsec:temporal_generation}.

\subsection{Training and Inference}
\label{subsec:train_test}
\method{} is trained in three stages: we first pretrain the FM, then train the CNP from the pretrained FM, and finally fine-tune the FM under the learned CNP. An overview of the staged training pipeline is illustrated in Fig.~\ref{fig:train}. Generation itself is straightforward: \method{} draws a source sample from the predicted source distribution and transforms it into a target sample via the learned vector field.

\paratitle{Stage 1: Pretraining the FM model.}
We pretrain the FM model in the standard manner for conditional generation, using a standard Gaussian $\mathcal{N}(\mathbf{0}, \mathbf{I})$ as its source distribution. The model takes the temporal condition $\mathcal{C_T}$ as an additional input and outputs a deterministic vector field $\bm{v}_\theta(\bm{x}_t, t \mid \mathcal{C_T})$. In temporal generation tasks, target samples under similar temporal conditions often share similar structures and dynamics. Since the learned ODE flow is continuous, flow reversal maps nearby targets to nearby regions in source space, yielding source samples that are clustered by the temporal condition rather than spread as a standard Gaussian. This suggests a feasible recipe: construct surrogate per-condition source distributions through flow reversal of the pretrained FM model, and train a model to predict them.

\paratitle{Stage 2: Training the conditional noise predictor.}
We train the CNP using source samples obtained through flow reversal of the pretrained FM model. For each temporal condition $\mathcal{C_T}$, the CNP parameterizes an isotropic Gaussian source distribution:
\begin{equation}
p_\phi(\bm{x}_0 \mid \mathcal{C_T})
:= \mathcal{N} \left(
\bm{x}_0; \bm{\mu}_\phi(\mathcal{C_T}), \sigma^2_\phi(\mathcal{C_T}) \mathbf{I}
\right),
\end{equation}
where $\phi$ denotes the CNP parameters, $\bm{\mu}_\phi(\mathcal{C_T})$ denotes the predicted mean, and
$\sigma_\phi(\mathcal{C_T})$ denotes the standard deviation. Given a target sample $\bm{x}_1 \sim p_1(\cdot \mid \mathcal{C_T})$, we retrieve its corresponding source
sample $\Tilde{\bm{x}}_0$ through \emph{flow reversal}:
\begin{equation}
  \Tilde{\bm{x}}_0
  = \Tilde{\Phi}_{\theta, 1 \to 0}(\bm{x}_1 \mid \mathcal{C_T}) := \bm{x}_1 -
  \int_0^1 \bm{v}_\theta \left(\bm{x}_\tau, \tau \mid \mathcal{C_T} \right)
  \mathrm{d}\tau,
\end{equation}
This mapping induces the conditional source distribution $\Tilde{p}_0(\cdot \mid \mathcal{C_T})$. 
The CNP model is trained to minimize two losses: (1) a negative log-likelihood (NLL) loss:
$\mathcal{L}_\text{NLL}(\phi) =
-\mathbb{E}_{\mathcal{C_T}, \Tilde{\bm{x}}_0 \sim \Tilde{p}_0(\cdot\mid\mathcal{C_T})
} \left[\log p_\phi(\Tilde{\bm{x}}_0 \mid \mathcal{C_T}) \right]$ that trains the CNP to approximate the flow-reversal-induced source distribution; and (2) a KL regularization loss: 
$\mathcal{L}_\text{KL}(\phi) = \mathcal{D}_\text{KL}(p_\phi(\cdot\mid \mathcal{C_T}) \| \mathcal{N}(\mathbf{0}, \mathbf{I}))$ that regularizes the learned conditional source distribution toward the standard Gaussian prior.
The total objective is:
\begin{equation}
  \mathcal{L}_{\mathrm{CNP}}(\phi) =
  \mathcal{L}_{\mathrm{NLL}}(\phi) +
  \lambda \mathcal{L}_{\mathrm{KL}}(\phi),
\end{equation}
where $\lambda > 0$ controls the strength of regularization.  Our choice of loss objectives is justified by the following propositions.  In Prop.~\ref{prop:cpn_min_nll}, we show that, with the NLL loss, the optimal output mean and output variance of the CNP model coincide with the mean and the scaled second central moment of the source distribution induced by flow reversal.

\begin{proposition}[(Informal) Population minimizer of NLL objective]~\label{prop:cpn_min_nll}
Let the CNP model the isotropic Gaussians $p_\phi(\bm{x}_0\mid\mathcal{C_T}) = \mathcal{N}(\bm{x}_0; \bm{\mu}_\phi(\mathcal{C_T}), \sigma_\phi^2(\mathcal{C_T}) \mathbf{I})$.
Then the population minimizer of $\mathcal{L}_{\mathrm{NLL}}$ over the isotropic Gaussian family is
$p_\phi^*(\bm{x}_0\mid\mathcal{C_T})
= \mathcal{N}(\bm{x}_0; \Tilde{\bm{\mu}}_0(\mathcal{C_T}), {\Tilde{\sigma}_0^2}(\mathcal{C_T}) \mathbf{I})$,
whose parameters satisfy
\begin{equation}
\begin{aligned}
\Tilde{\bm{\mu}}_0(\mathcal{C_T}) &= \mathbb E_{\bm{x}_1 \sim p_1(\cdot \mid \mathcal{C_T})}[\Tilde{\Phi}_{\theta, 1 \to 0}(\bm{x}_1 \mid \mathcal{C_T})], \text{  and  } \\
{\Tilde{\sigma}_0^2}(\mathcal{C_T}) &= \frac{1}{d} \mathbb E_{\bm{x}_1 \sim p_1(\cdot \mid \mathcal{C_T})}
\|\Tilde{\Phi}_{\theta, 1 \to 0}(\bm{x}_1 \mid \mathcal{C_T}) - \Tilde{\bm{\mu}}_0\|_2^2.
\end{aligned}
\end{equation}
\end{proposition}

The propositions covering more general cases, along with their proofs, are provided in Appendix~\ref{app:mathematical_details}. In particular, Prop.~\ref{prop:cnp_min} establishes the results considering both the NLL and the KL divergence loss, while Cor.~\ref{cor:cnp_limits} characterizes the limiting behavior of the distribution in extreme cases.  We also summarize the entire training procedure of the CNP model in Algo.~\ref{alg:train_cnp} in Appendix~\ref{app:algorithms}.

\paratitle{Stage 3: Fine-tuning the FM model with predicted conditional source distributions.}
Lastly, we fine-tune the FM model with the per-condition source distribution predicted by the CNP model.  This allows us to refine its underlying vector field and further improve generation performance. Following the same training procedure as in the first stage, we replace the standard Gaussian with the predicted distributions for each condition.
\begin{equation}
\mathcal{L}_{\mathrm{FT-FM}}(\theta) = \mathbb{E}_{\mathcal{C_T}, \bm{x}_0 \sim p_\phi(\cdot\mid\mathcal{C_T}), \bm{x}_1 \sim p_1(\cdot \mid \mathcal{C_T}), t \sim \mathcal{U}(0,1)} \left[\|\bm{v}_\theta(\bm{x}_t, t \mid \mathcal{C_T}) - (\bm{x}_1 - \bm{x}_0)\|_2^2 \right].
\end{equation}
We summarize the fine-tuning procedure in Algo.~\ref{alg:train_cfm} in Appendix~\ref{app:algorithms}.

\paratitle{Inference with \method{}.}
During inference, we first draw source samples with the CNP model $\bm{x}_0 \sim p_\phi(\cdot\mid\mathcal{C_T})$, and transform them into target data with the FM model (Eqn.~\ref{eqn:fm_inference}).


\begin{theorem}[(Informal) Conditional source locality]
Suppose the predicted source mean $\bm{\mu}_\phi^*(\cdot)$ is $L_\mu$-Lipschitz
continuous in the temporal condition.
For two conditions $\mathcal{C}_\mathcal{T}$ and $\mathcal{C}_\mathcal{T}'$,
independently drawn CNP source samples satisfy,
with high probability,
\begin{equation}
\|\bm{x}_0 - \bm{x}_0'\|_2
\leq
L_\mu\|\mathcal{C}_\mathcal{T} - \mathcal{C}_\mathcal{T}'\|_2
+ \mathcal{O}\left(\sqrt{d\bigl(\sigma_\phi^{*2}(\mathcal{C}_\mathcal{T})+\sigma_\phi^{*2}(\mathcal{C}_\mathcal{T}') \bigr)}\right).
\end{equation}
\end{theorem}

This theorem characterizes locality in the learned source space. Under a Lipschitz condition on the predicted mean, nearby temporal conditions yield nearby source centers, while the predicted variances control stochastic dispersion around these centers. The formal statement and proof are provided in Appendix~\ref{app:mathematical_details}.

\subsection{Application to Temporal Generation Tasks}
\label{subsec:temporal_generation}
We consider a series of temporal generation tasks that progressively generate sequential observations given temporal conditions. This formulation covers video prediction, video interpolation, and motion planning, which differ primarily in how the temporal condition is constructed.

\paratitle{Video Prediction.}
Given initial frames $\bm{x}^0$ and $\bm{x}^1$, the goal is to autoregressively predict future frames $\bm{x}^i$ for $i = 2, \ldots, N$. At step $i$, the preceding frame $\bm{x}^{i-1}$ serves as the reference frame, while a context frame $\bm{x}^c$ is sampled from earlier observations with $c \sim \mathcal{U}(\{0, 1, \ldots, i-2\})$, yielding the temporal condition $\mathcal{C}_\mathcal{T}^i = (\bm{x}^{i-1}, \bm{x}^c, i-c)$ that captures both recent dynamics and longer-range temporal structure.

\paratitle{Video Interpolation.}
Given boundary frames $\bm{x}^0$ and $\bm{x}^N$, the goal is to progressively predict intermediate frames $\bm{x}^i$ for $i = 1, \ldots, N-1$. At step $i$, the previously generated frame $\bm{x}^{i-1}$ serves as the reference, while $\bm{x}^N$ provides a fixed future context, giving $\mathcal{C}_\mathcal{T}^i = (\bm{x}^{i-1}, \bm{x}^N, N-i)$, where $N-i$ denotes the remaining temporal distance to the endpoint.


\paratitle{Motion Planning.}
Given an initial state $\bm{x}^0$ and a target end-effector
pose $\bm{x}^N$, the objective is to generate a feasible
trajectory that reaches the target pose from the initial state.
The model generates the entire trajectory jointly,
conditioned on $\mathcal{C}_\mathcal{T} = (\bm{x}^0, \bm{x}^N)$.

Detailed sampling procedures and algorithms are provided in Appendix~\ref{app:algorithms}.

\section{Experiments}
\label{sec:experiments}

\subsection{Experimental Setup}
\label{sec:experimental_setup}
We evaluate \method{} across two modalities: \emph{video generation}, comprising prediction and interpolation, where output quality emphasizes visual coherence; and \emph{motion planning}, where the learned, condition-dependent source is expected to generalize from pixels to robot trajectories.

\paratitle{Video generation.}
We evaluate on BAIR Robot Pushing~\citep{ebert2017self} ($64 \times 64$), KTH actions~\citep{schuldt2004recognizing} ($64 \times 64$), and CLEVRER~\citep{yi2020clevrer} ($128 \times 128$), spanning stochastic manipulation, structured human movement, and deterministic object dynamics. Following RIVER~\citep{davtyan2023river}, we use the SRVP splits~\citep{franceschi2020stochastic} and report PSNR, SSIM~\citep{wang2004ssim}, and FVD~\citep{unterthiner2019fvd} under the SRVP/RIVER protocol. 
We define prediction as $p \!\rightarrow\! k$, where $p$ observed frames are used to autoregressively generate $k$ future frames. For interpolation, $p + f \!\rightarrow\! k$ denotes using $p$ initial and $f$ terminal observed frames to generate $k$ intermediate frames; for brevity we write $2 \!\rightarrow\! 14$ for the two-endpoint case $1{+}1 \!\rightarrow\! 14$. Our primary baseline is RIVER, the fixed-source FM backbone we replace, providing a direct comparison between the standard Gaussian source and our learned, condition-dependent source. For other video-generation baselines, including MAGVIT~\citep{yu2023magvit}, iVideoGPT~\citep{wu2024ivideogpt}, CVF~\citep{shrivastava2024cvf}, and PFI~\citep{chen2024pfi}, we report results from the original papers. We further compare with recent methods that modify or learn the source distribution, including CAR-Flow~\citep{chen2025car}, PYoCo~\citep{ge2023preserve}, and CSFM~\citep{kim2026better}. Since directly comparable results are not available under our BAIR and KTH settings, we reimplement these methods following their original formulations and evaluate them under matched backbone, training-budget, and evaluation settings. Complete implementation and evaluation details are provided in Appendix~\ref{app:experimental_config}.

\begin{table}[t]
\centering
\footnotesize
\renewcommand{\arraystretch}{0.8}
\setlength{\tabcolsep}{6pt}
\caption{\textbf{Results on video prediction.}
CNP-Flow improves over the fixed-source RIVER baseline and remains competitive with recent source-learning methods. CNP-Flow results are averaged over three seeds.}
\label{tab:video_prediction}
\begin{tabular}{c c l c c c}
\toprule
Dataset & $p \rightarrow k$ & Method
& PSNR $\uparrow$ & SSIM $\uparrow$ & FVD $\downarrow$ \\
\midrule

\multirow[c]{7}{*}{BAIR64}
& \multirow[c]{7}{*}{$1 \rightarrow 15$}
& RIVER~\citep{davtyan2023river}
    & -- & -- & 73.50 \\
& & MAGVIT~\citep{yu2023magvit}
    & 19.30 & 0.787 & 62.00 \\
& & iVideoGPT~\citep{wu2024ivideogpt}
    & \textbf{20.40} & 0.823 & 75.00 \\
& & CAR-Flow~\citep{chen2025car}
    & 20.08 & 0.836 & 71.37 \\ 
& & PYoCo~\citep{ge2023preserve}
    & 20.13 & \second{0.837} & \second{61.88} \\
& & CSFM~\citep{kim2026better}
    & \second{20.16} & 0.836 & 64.20 \\
& & \cellcolor{ourshl}\textbf{\method{} (Ours)}
    & \cellcolor{ourshl}20.15
    & \cellcolor{ourshl}\textbf{0.839}
    & \cellcolor{ourshl}\textbf{61.17} \\

\midrule

\multirow[c]{6}{*}{KTH}
& \multirow[c]{6}{*}{$10 \rightarrow 30$}
& RIVER~\citep{davtyan2023river}
    & 30.40 & 0.860 & 180.00 \\
& & CVF~\citep{shrivastava2024cvf}
    & 30.60 & 0.891 & 108.60 \\
& & CAR-Flow~\citep{chen2025car}
    & 29.25 & 0.897 & 126.64 \\ 
& & PYoCo~\citep{ge2023preserve}
    & 31.14 & 0.903 & 82.32 \\
& & CSFM~\citep{kim2026better}
    & \second{31.64} & \second{0.915} & \second{69.00} \\ 
& & \cellcolor{ourshl}\textbf{\method{} (Ours)}
    & \cellcolor{ourshl}\textbf{34.25}
    & \cellcolor{ourshl}\textbf{0.943}
    & \cellcolor{ourshl}\textbf{44.76} \\ 

\midrule

\multirow[c]{3}{*}{KTH}
& \multirow[c]{3}{*}{$10 \rightarrow 40$}
& RIVER~\citep{davtyan2023river}
    & 29.00 & 0.820 & 170.50 \\
& & CVF~\citep{shrivastava2024cvf}
    & \second{29.70} & \second{0.852} & \second{100.80} \\
& & \cellcolor{ourshl}\textbf{\method{} (Ours)}
    & \cellcolor{ourshl}\textbf{32.69}
    & \cellcolor{ourshl}\textbf{0.930}
    & \cellcolor{ourshl}\textbf{54.88} \\ 

\midrule

\multirow[c]{3}{*}{CLEVRER}
& \multirow[c]{3}{*}{$2 \rightarrow 14$}
& RIVER~\citep{davtyan2023river}
    & -- & -- & 48.96 \\
& & PFI~\citep{chen2024pfi}
    & -- & -- & \second{39.31} \\
& & \cellcolor{ourshl}\textbf{\method{} (Ours)}
    & \cellcolor{ourshl}\textbf{42.23}
    & \cellcolor{ourshl}\textbf{0.994}
    & \cellcolor{ourshl}\textbf{24.18} \\

\bottomrule
\end{tabular}
\vspace{1em}
\centering
\footnotesize
\renewcommand{\arraystretch}{0.9}
\setlength{\tabcolsep}{8pt}
\caption{\textbf{Results on video interpolation.} CNP-Flow improves all metrics over RIVER on both datasets.}
\label{tab:video_interpolation}
\begin{tabular}{c c l c c c}
\toprule
Dataset & $p\!+\!f \rightarrow k$ & Method & PSNR $\uparrow$ & SSIM $\uparrow$ & FVD $\downarrow$ \\
\midrule
\multirow[c]{2}{*}{KTH}
& \multirow[c]{2}{*}{$2 \rightarrow 14$}
& RIVER~\citep{davtyan2023river}        & 35.65 & 0.9627 & 38.28 \\
& & \cellcolor{ourshl}\textbf{\method{} (Ours)} & \cellcolor{ourshl}\textbf{36.07} & \cellcolor{ourshl}\textbf{0.9634} & \cellcolor{ourshl}\textbf{36.81} \\
\midrule
\multirow[c]{2}{*}{CLEVRER}
& \multirow[c]{2}{*}{$2 \rightarrow 14$}
& RIVER~\citep{davtyan2023river}        & 43.79 & 0.9949 & 16.01 \\
& & \cellcolor{ourshl}\textbf{\method{} (Ours)} & \cellcolor{ourshl}\textbf{45.68} & \cellcolor{ourshl}\textbf{0.9954} & \cellcolor{ourshl}\textbf{10.77} \\
\bottomrule
\end{tabular}
\end{table}

\paratitle{Motion planning.}
For cross-modality transfer, we evaluate on the 7-DoF Franka Panda benchmark from MPD~\citep{carvalho2025motion}: each sample is a joint-space trajectory conditioned on start and goal configurations. We compare \method{} against MPD (a diffusion-based planner) and a 
fixed-source FM, all trained on identical trajectory data. Following MPD, we report success, valid fraction, end-effector position/orientation error, and collision intensity. FM-based methods use a $15$-step Euler ODE sampler followed by MPD post-hoc cost refinement; details are provided in Appendix~\ref{app:experimental_config}.

\subsection{Main results}
\label{sec:main-results}

\paratitle{Video generation.} 
\method{} achieves the best FVD and SSIM across all prediction settings (Tab.~\ref{tab:video_prediction}), consistently improving over the fixed-source RIVER backbone and the strongest competing baselines. The gain is particularly large on KTH $10\!\rightarrow\!30$, where FVD decreases from $180.00$ with RIVER to $44.76$ with \method{}, compared with $126.64$, $82.32$, and $69.00$ for CAR-Flow, PYoCo, and CSFM, respectively, while the gap is much smaller on BAIR. These methods either use predefined source structures or learn the source jointly with the flow using regularization or alignment. In contrast, \method{} learns its conditional source from condition-specific source samples obtained through flow reversal, providing explicit supervision rather than relying on predefined structure or indirect joint optimization.

The same trend holds for $2\!\rightarrow\!14$ interpolation (Tab.~\ref{tab:video_interpolation}), with consistent gains on all metrics over RIVER. Interpolation gains are smaller in absolute terms, as both endpoints constrain the rollout and reduce the burden on the source distribution.

\begin{table}[t]
\caption{\textbf{Results on motion planning.}
All methods use a matched 15-step sampling budget, and results are
averaged over three seeds.}

\centering
\footnotesize
\setlength{\tabcolsep}{4pt}

\begin{tabular}{l c c c c c}
\toprule
Method
& Success $\uparrow$
& Valid (\%) $\uparrow$
& EE pos (cm) $\downarrow$
& EE ori ($^\circ$) $\downarrow$
& Collision $\downarrow$ \\
\midrule

MPD~\citep{carvalho2025motion}
& \textbf{97\%}
& 86.6
& 1.9
& 1.2
& 0.0030 \\

Fixed-source FM
& 96\%
& 84.9
& 1.6
& 1.1
& 0.0038 \\

\rowcolor{ourshl}
\textbf{\method{} (Ours)}
& \textbf{97\%}
& \textbf{88.5}
& \textbf{0.8}
& \textbf{0.5}
& \textbf{0.0023} \\

\bottomrule
\end{tabular}

\label{tab:motion_planning}
\end{table}
\begin{figure}[h]
\centering
\includegraphics[width=0.8\textwidth]{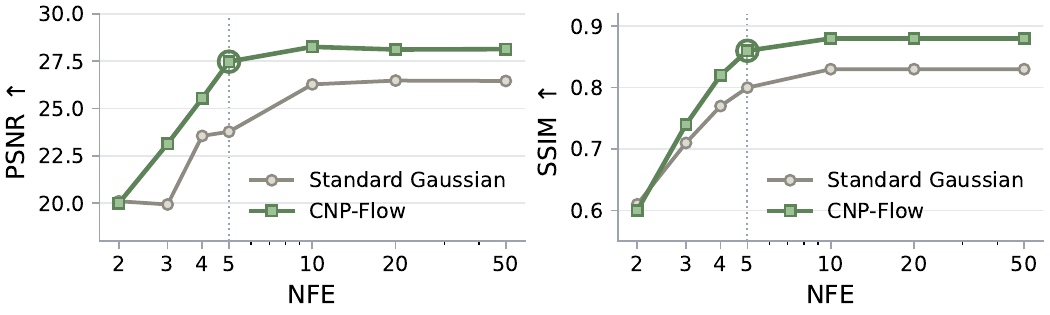}
\caption{\textbf{Inference efficiency on KTH prediction.}
Curves show PSNR/SSIM averaged across generated samples per video. \method{} reaches the standard Gaussian PSNR plateau ($\approx 26.5$) at NFE $=5$ and converges higher ($\approx 28.2$) by NFE $=10$; SSIM shows a similar trend. The crossover point is highlighted on each panel.}
\label{fig:efficiency}
\end{figure}

\paratitle{Motion planning.}
The learned source is modality-agnostic: it parameterizes an isotropic Gaussian over latents conditioned on the temporal condition, regardless of what those latents represent. On the 7-DoF Franka Panda benchmark (Tab.~\ref{tab:motion_planning}), replacing diffusion (MPD) with a standard Gaussian source FM of similar capacity changes no metric outside noise. This rules out the generator family as the active variable. Adding the learned source to the same FM raises Valid from $84.9\%$ to $88.5\%$, halves both end-effector errors, and cuts collision intensity by about $40\%$. Success rate moves only $96\% \!\rightarrow\! 97\%$ because this metric ("at least one valid trajectory in the batch") is already saturated. Valid counts the fraction of sampled trajectories that are feasible, which is precisely what a better source distribution should improve. The mechanism matches the video case: the learned source places the integration start aligned with the manifold of valid outputs. The modality, dimension, and failure mode (a robot collision rather than low SSIM) are different.

\paratitle{Inference efficiency.}
Because $\bm{\mu}_\phi(\mathcal{C}^i)$ is already aligned with the data manifold, the ODE no longer needs to first traverse the gap from a standard Gaussian before refining—each step is local refinement, so coarse solvers suffice. Fig.~\ref{fig:efficiency} substantiates this on the KTH prediction. Standard Gaussian source baseline plateaus at PSNR $\approx 26.5$ from NFE $=10$ onward; \method{} attains and surpasses that plateau at NFE $=5$ (PSNR $27.48$, SSIM $0.86$) and continues rising to PSNR $\approx 28.2$ at NFE $=10$. \method{} at NFE $=5$ already exceeds the baseline at NFE $=50$, an order-of-magnitude reduction in solver budget; the final quality is superior, so the gain is both efficiency and endpoint improvement.

\subsection{Analysis}
\label{sec:analysis}

\paratitle{Qualitative comparison.}
Fig.~\ref{fig:visualize} shows the effect of the learned source. On CLEVRER interpolation, the fixed-source FM baseline incurs appearance drift and ghosting, while \method{} maintains object identity and tracks the ground truth. On Franka Panda, MPD and fixed-source FM produce numerous colliding trajectories, whereas \method{} produces a denser set of valid, collision-free plans. The visual failure modes reflect the metric results: a condition-dependent source improves temporal consistency in video and trajectory validity in motion planning.

\begin{figure}[!t]
    \centering
    \includegraphics[width=0.95\linewidth]{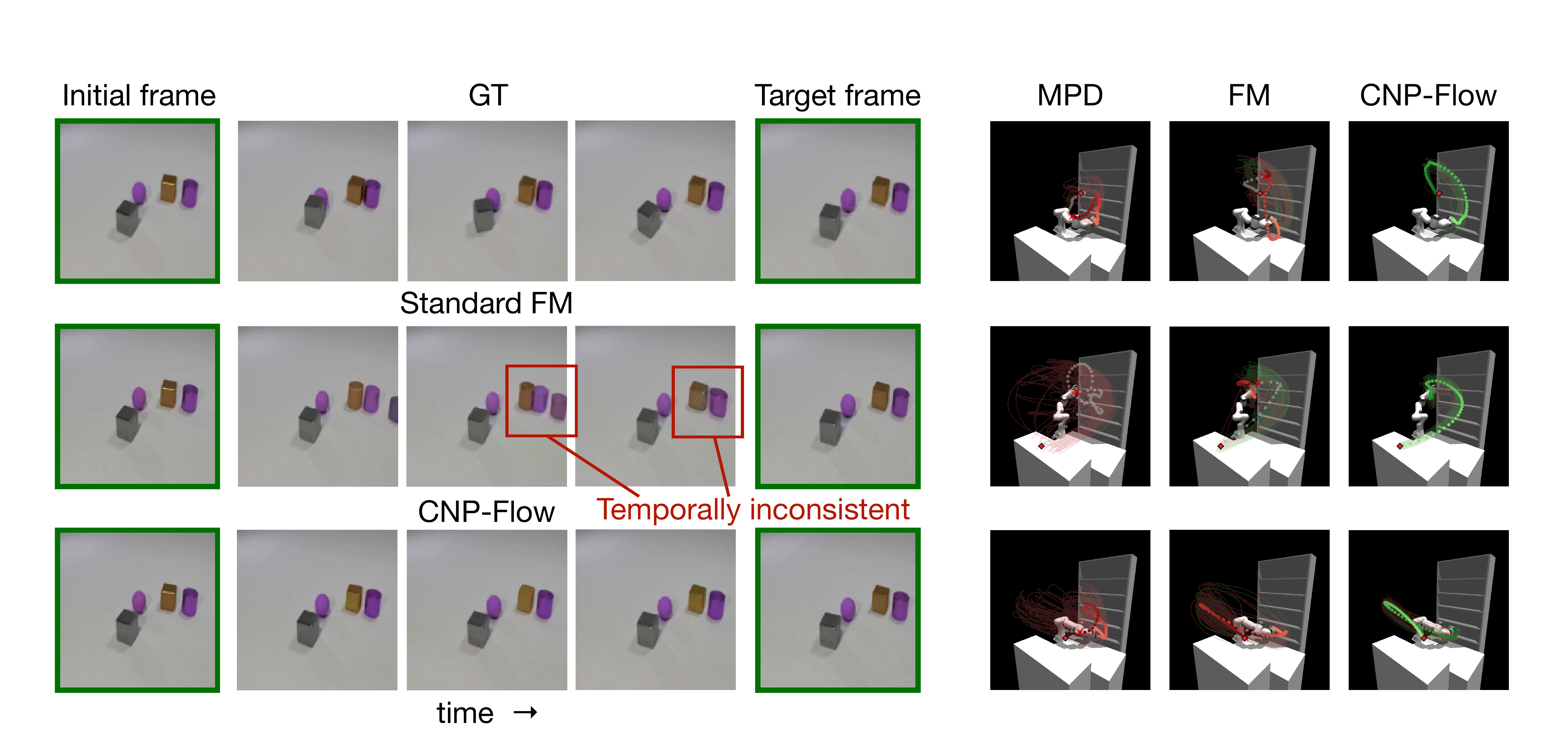}
    \caption{\textbf{Qualitative comparison.} \emph{Left}: CLEVRER interpolation. Fixed-source FM drifts at the highlighted frames, while \method{} matches the ground truth. \emph{Right}: Franka Panda motion planning. Each panel overlays 50 sampled trajectories (red: colliding, green: collision-free). \method{} produces more collision-free plans than MPD and fixed-source FM.
    }
    \label{fig:visualize}
\end{figure}

\begin{table}[!t]
\centering
\caption{\textbf{Ablation studies.} \emph{Left}: source parameterization on KTH $10\!\rightarrow\!30$ video prediction, isolating the contribution of the learned mean and learned variance. \emph{Right}: source learning and backbone fine-tuning (FT) on Franka Panda motion planning.
In both tasks, metrics improve as the source better matches the condition.
}
\setlength{\tabcolsep}{4pt}
\renewcommand{\arraystretch}{0.95}
\resizebox{\textwidth}{!}{%
\footnotesize
\begin{tabular}{l c c c}
\toprule
\multicolumn{4}{c}{Video prediction (KTH)} \\
\midrule
Source $\bm{x}_0$ & PSNR $\uparrow$ & SSIM $\uparrow$ & FVD $\downarrow$ \\
\midrule
$\mathcal{N}(\bm{0}, \mathbf{I})$
  & 30.40 & 0.860  & 180.00 \\
$\bm{\mu}_\phi(\mathcal{C}^i)$
  & 31.46 & 0.908 & 83.82  \\
\rowcolor{ourshl}
$\mathcal{N}(\bm{\mu}_\phi(\mathcal{C}^i), \bm{\sigma}_\phi^2(\mathcal{C}^i)\mathbf{I})$
  & \textbf{34.25} & \textbf{0.943} & \textbf{44.76} \\
\bottomrule
\end{tabular}%
\hspace{5pt}%
\begin{tabular}{l c c c c}
\toprule
\multicolumn{5}{c}{Motion planning (Franka Panda)} \\
\midrule
Variant & Valid (\%) $\uparrow$ & EE pos (cm) $\downarrow$ & EE ori ($^\circ$) $\downarrow$ & Coll. $\downarrow$ \\
\midrule
Standard Gaussian
  & 84.9 & 1.6 & 1.1 & 0.0038 \\
\method{} w/o FT
  & 87.5 & 1.0 & 0.7 & 0.0024 \\
\rowcolor{ourshl}
\method{} w/ FT
  & \textbf{88.5} & \textbf{0.8} & \textbf{0.5} & \textbf{0.0023} \\
\bottomrule
\end{tabular}%
}
\label{tab:ablation_combined}
\end{table}

\paratitle{Ablation studies.}
We ablate two design choices of \method{} in Tab.~\ref{tab:ablation_combined}: the parameterization of the learned source for video prediction, and whether the FM backbone should be fine-tuned after introducing the learned source. On KTH $10\!\rightarrow\!30$, replacing the standard Gaussian source with the deterministic learned mean $\bm{\mu}_\phi(\mathcal{C}^i)$ improves FVD from $180.00$ to $83.82$, showing that adapting the source to the temporal condition already provides a large gain over an unconditional source. However, the deterministic mean restricts each condition to a single source point. Adding the learned variance to form the full source $\bm{\mu}_\phi(\mathcal{C}^i)+\bm{\sigma}_\phi(\mathcal{C}^i)\bm{\epsilon}$ improves FVD to $44.76$, PSNR from $31.46$ to $34.25$, and SSIM from $0.908$ to $0.943$. This shows that stochasticity remains important for video prediction, as multiple plausible futures can exist under the same temporal conditions. Additional controls in Appendix~\ref{app:additional_analysis} show that, with the FM frozen, a CNP without flow reversal supervision performs similarly to the standard Gaussian source, whereas flow reversal supervision substantially improves FVD, isolating its contribution to source learning.

For motion planning, the right side of Tab.~\ref{tab:ablation_combined} isolates two effects. Replacing the standard Gaussian source with a learned source (rows $1\!\rightarrow\!2$) yields the larger effect: Valid jumps from $84.9\%$ to $87.5\%$, end-effector position error drops from $1.6$ cm to $1.0$ cm, and collision intensity drops by about $37\%$. All of this occurs without touching the FM backbone. Fine-tuning the backbone on the new source (rows $2\!\rightarrow\!3$) closes the remaining gap: Valid $87.5\%\!\rightarrow\!88.5\%$, orientation $0.7^\circ\!\rightarrow\!0.5^\circ$, collision $0.0024\!\rightarrow\!0.0023$. This suggests that re-aligning the velocity field with the learned, condition-dependent source distribution adds refinement.



\paratitle{Limitations and future work.}\label{sec:limitations}
While CNP-Flow improves generation quality and sampling efficiency across video prediction, video interpolation, and motion planning, one key limitation remains: the CNP uses a fixed form of temporal condition per task, limiting its use of richer contexts such as longer-horizon histories. A natural extension is to pair more expressive conditional source families with a flexible context encoder that absorbs longer-range, task-dependent structure. More broadly, treating the source distribution as a learnable conditioning channel is orthogonal to vector-field conditioning and applicable to any conditional FM task.
\section{Conclusion}
We presented CNP-Flow, a framework for learning conditional
source distributions through flow reversal for temporal generation. By using a pretrained flow model to provide source supervision, CNP-Flow learns a source distribution adapted to each temporal condition while preserving the FM backbone architecture. Experiments across video prediction, video interpolation, and 7-DoF Franka motion planning demonstrate improved generation quality, alongside gains in sampling efficiency.


\section*{Acknowledgment}
This work was supported by the AMD University Program AI \& HPC Cluster. We further acknowledge Kuo-Guang Tsai for his technical support on the AMD system cluster infrastructure. This research was also supported by the National Science and Technology Council (NSTC), Taiwan, under Grants 115-2222-E-002-006-, 115-2634-F-002-001-,  115-2218-E-002-026-, and 115-2634-F-002-012-.

\section*{AI Use Disclosure Statement}
Large language models were used to assist with language editing and improve the clarity of the manuscript. Claude Code was also used to support software development, including implementation and debugging. All AI-generated text and code were reviewed, revised as needed, and validated by the authors, who take full responsibility for the final manuscript and code.

\section*{Reproducibility Statement}
We provide the details necessary to reproduce our results in the main paper and appendix. The complete training and inference procedures of \method{} are provided in Appendix~\ref{app:algorithms}, with additional mathematical details in Appendix~\ref{app:mathematical_details}. Appendix~\ref{app:experimental_config} specifies the experimental setup, model architectures, training hyperparameters, flow reversal configurations, computational resources, and evaluation protocols for both video generation and motion planning. We also report results over multiple independent random seeds where applicable and provide additional analyses and generated results in Appendices~\ref{app:additional_analysis} and~\ref{app:additional_results}.

{
    \small
    \bibliographystyle{ieeenat_fullname}
    \bibliography{bibs/main}

@STRING{NeurIPS = "Advances in Neural Information Processing Systems (NeurIPS)"}

@STRING{NIPS = "Advances in Neural Information Processing Systems (NIPS)"}

@STRING{ICML = "International Conference on Machine Learning (ICML)"}

@STRING{ICLR = "International Conference on Learning Representations (ICLR)"}

@STRING{AAAI = "Association for the Advancement of Artificial Intelligence (AAAI)"}

@STRING{CVPR = "Proceedings of the IEEE/CVF Conference on Computer Vision and Pattern Recognition (CVPR)"}

@STRING{ECCV = "Proceedings of the European Conference on Computer Vision (ECCV)"}

@STRING{ICCV = "Proceedings of the IEEE/CVF International Conference on Computer Vision (ICCV)"}

@STRING{CORL = "Conference on Robot Learning (CoRL)"}

@STRING{TMLR = "Transactions on Machine Learning Research (TMLR)"}

@inproceedings{ho2020denoising,
  title     = {Denoising Diffusion Probabilistic Models},
  author    = {Ho, Jonathan and Jain, Ajay and Abbeel, Pieter},
  booktitle = NeurIPS,
  year      = {2020}
}

@inproceedings{lipman2023flow,
  title     = {Flow Matching for Generative Modeling},
  author    = {Lipman, Yaron and Chen, Ricky T. Q. and Ben-Hamu, Heli and Nickel, Maximilian and Le, Matt},
  booktitle = ICLR,
  year      = {2023},
}

@article{tong2024improving,
  title   = {Improving and Generalizing Flow-Based Generative Models with Minibatch Optimal Transport},
  author  = {Tong, Alexander and Fatras, Kilian and Malkin, Nikolay and Huguet, Guillaume and Zhang, Yanlei and Rector-Brooks, Jarrid and Wolf, Guy and Bengio, Yoshua},
  journal = TMLR,
  year    = {2024},
}

@inproceedings{pooladian2023multisample,
  title     = {Multisample Flow Matching: Straightening Flows with Minibatch Couplings},
  author    = {Pooladian, Aram-Alexandre and Ben-Hamu, Heli and Domingo-Enrich, Carles and Amos, Brandon and Lipman, Yaron and Chen, Ricky T. Q.},
  booktitle = ICML,
  year      = {2023}
}

@inproceedings{esser2024scaling,
  title     = {Scaling Rectified Flow Transformers for High-Resolution Image Synthesis},
  author    = {Esser, Patrick and Kulal, Sumith and Blattmann, Andreas and Entezari, Rahim and M{\"u}ller, Jonas and Saini, Harry and Levi, Yam and Lorenz, Dominik and Sauer, Axel and Boesel, Frederic and Podell, Dustin and Dockhorn, Tim and English, Zion and Rombach, Robin},
  booktitle = ICML,
  year      = {2024},
}

@inproceedings{3d_flowmatch_actor,
  author = {Gkanatsios, Nikolaos and Xu, Jiahe and Bronars, Matthew and Mousavian, Arsalan and Ke, Tsung-Wei and Fragkiadaki, Katerina},
  title = {3D FlowMatch Actor: Unified 3D Policy for Single- and Dual-Arm Manipulation},
  booktitle = {arXiv preprint arXiv:2508.11002},
  year = {2025},
}

@inproceedings{wagenmaker2025steering,
  title={Steering your diffusion policy with latent space reinforcement learning},
  author={Wagenmaker, Andrew and Nakamoto, Mitsuhiko and Zhang, Yunchu and Park, Seohong and Yagoub, Waleed and Nagabandi, Anusha and Gupta, Abhishek and Levine, Sergey},
  booktitle=CORL,
  year={2025}
}

@inproceedings{voleti2022mcvd,
  title={Mcvd-masked conditional video diffusion for prediction, generation, and interpolation},
  author={Voleti, Vikram and Jolicoeur-Martineau, Alexia and Pal, Chris},
  booktitle=NeurIPS,
  year={2022},
}

@inproceedings{yu2023magvit,
  title={Magvit: Masked generative video transformer},
  author={Yu, Lijun and Cheng, Yong and Sohn, Kihyuk and Lezama, Jos{\'e} and Zhang, Han and Chang, Huiwen and Hauptmann, Alexander G and Yang, Ming-Hsuan and Hao, Yuan and Essa, Irfan and others},
  booktitle=CVPR,
  year={2023},
}

@inproceedings{wu2024ivideogpt,
  title={ivideogpt: Interactive videogpts are scalable world models},
  author={Wu, Jialong and Yin, Shaofeng and Feng, Ningya and He, Xu and Li, Dong and Hao, Jianye and Long, Mingsheng},
  booktitle=NeurIPS,
  year={2024},
}

@inproceedings{davtyan2023river,
  title={Efficient video prediction via sparsely conditioned flow matching},
  author={Davtyan, Aram and Sameni, Sepehr and Favaro, Paolo},
  booktitle=ICCV,
  year={2023},
}

@inproceedings{shrivastava2024cvp,
  title={Video prediction by modeling videos as continuous multi-dimensional processes},
  author={Shrivastava, Gaurav and Shrivastava, Abhinav},
  booktitle=CVPR,
  year={2024},
}

@article{shrivastava2024cvf,
  title={Efficient Continuous Video Flow Model for Video Prediction},
  author={Shrivastava, Gaurav and Shrivastava, Abhinav},
  journal={arXiv preprint arXiv:2412.05633},
  year={2024},
}

@inproceedings{chen2024pfi,
  title={Probabilistic Forecasting with Stochastic Interpolants and F{\"o}llmer Processes},
  author={Chen, Yifan and Goldstein, Mark and Hua, Mengjian and Albergo, Michael Samuel and Boffi, Nicholas Matthew and Vanden-Eijnden, Eric},
  booktitle=ICML,
  year={2024},
}

@article{hoppe2022diffusion,
  title={Diffusion models for video prediction and infilling},
  author={H{\"o}ppe, Tobias and Mehrjou, Arash and Bauer, Stefan and Nielsen, Didrik and Dittadi, Andrea},
  journal=TMLR,
  year={2022}
}

@inproceedings{ho2022video,
  title={Video diffusion models},
  author={Ho, Jonathan and Salimans, Tim and Gritsenko, Alexey and Chan, William and Norouzi, Mohammad and Fleet, David J},
  booktitle=NeurIPS,
  year={2022}
}

@inproceedings{franceschi2020stochastic,
  title={Stochastic latent residual video prediction},
  author={Franceschi, Jean-Yves and Delasalles, Edouard and Chen, Micka{\"e}l and Lamprier, Sylvain and Gallinari, Patrick},
  booktitle=ICML,
  year={2020},
}

@inproceedings{ebert2017self,
  title     = {Self-Supervised Visual Planning with Temporal Skip Connections},
  author    = {Ebert, Frederik and Finn, Chelsea and Lee, Alex X. and Levine, Sergey},
  booktitle = CORL,
  year      = {2017}
}

@inproceedings{schuldt2004recognizing,
  title     = {Recognizing Human Actions: A Local {SVM} Approach},
  author    = {Sch{\"u}ldt, Christian and Laptev, Ivan and Caputo, Barbara},
  booktitle = {Proceedings of the 17th International Conference on Pattern Recognition (ICPR)},
  year      = {2004},
}

@inproceedings{yi2020clevrer,
  title     = {{CLEVRER}: Collision Events for Video Representation and Reasoning},
  author    = {Yi, Kexin and Gan, Chuang and Li, Yunzhu and Kohli, Pushmeet and Wu, Jiajun and Torralba, Antonio and Tenenbaum, Joshua B.},
  booktitle = ICLR,
  year      = {2020}
}

@article{wang2004ssim,
  title   = {Image Quality Assessment: From Error Visibility to Structural Similarity},
  author  = {Wang, Zhou and Bovik, Alan C. and Sheikh, Hamid R. and Simoncelli, Eero P.},
  journal = {IEEE Transactions on Image Processing},
  year    = {2004},
}

@article{unterthiner2019fvd,
  title   = {{FVD}: A New Metric for Video Generation},
  author  = {Unterthiner, Thomas and van Steenkiste, Sjoerd and Kurach, Karol and Marinier, Rapha{\"e}l and Michalski, Marcin and Gelly, Sylvain},
  journal = {ICLR Workshop on Deep Generative Models for Highly Structured Data},
  year    = {2019}
}

@inproceedings{eyring2024reno,
  title={Reno: Enhancing one-step text-to-image models through reward-based noise optimization},
  author={Eyring, Luca and Karthik, Shyamgopal and Roth, Karsten and Dosovitskiy, Alexey and Akata, Zeynep},
  booktitle=NeurIPS,
  year={2024}
}

@inproceedings{samuel2024generating,
  title={Generating images of rare concepts using pre-trained diffusion models},
  author={Samuel, Dvir and Ben-Ari, Rami and Raviv, Simon and Darshan, Nir and Chechik, Gal},
  booktitle=AAAI,
  year={2024}
}

@inproceedings{mao2024lottery,
  title={The lottery ticket hypothesis in denoising: Towards semantic-driven initialization},
  author={Mao, Jiafeng and Wang, Xueting and Aizawa, Kiyoharu},
  booktitle=ECCV,
  year={2024},
}

@article{mao2023semantic,
  title={Semantic-driven initial image construction for guided image synthesis in diffusion model},
  author={Mao, Jiafeng and Wang, Xueting and Aizawa, Kiyoharu},
  journal={arXiv preprint arXiv:2312.08872},
  year={2023}
}

@article{chiang2026cot,
  title={COT-FM: Cluster-wise Optimal Transport Flow Matching},
  author={Chiang, Chiensheng and Tu, Kuan-Hsun and Liao, Jia-Wei and Chou, Cheng-Fu and Ke, Tsung-Wei},
  journal={arXiv preprint arXiv:2603.13395},
  year={2026}
}

@inproceedings{chang2025warped,
  title={How i warped your noise: a temporally-correlated noise prior for diffusion models},
  author={Chang, Pascal and Tang, Jingwei and Gross, Markus and Azevedo, Vinicius C},
  booktitle=ICLR,
  year={2024}
}

@inproceedings{burgert2025go,
  title={Go-with-the-flow: Motion-controllable video diffusion models using real-time warped noise},
  author={Burgert, Ryan and Xu, Yuancheng and Xian, Wenqi and Pilarski, Oliver and Clausen, Pascal and He, Mingming and Ma, Li and Deng, Yitong and Li, Lingxiao and Mousavi, Mohsen and others},
  booktitle=CVPR,
  year={2025}
}

@article{liu2025equivdm,
  title={Equivdm: Equivariant video diffusion models with temporally consistent noise},
  author={Liu, Chao and Vahdat, Arash},
  journal={arXiv preprint arXiv:2504.09789},
  year={2025}
}

@article{kim2026better,
  title={Better Source, Better Flow: Learning Condition-Dependent Source Distribution for Flow Matching},
  author={Kim, Junwan and Park, Jiho and Jeon, Seonghu and Kim, Seungryong},
  journal={arXiv preprint arXiv:2602.05951},
  year={2026}
}

@article{lee2025there,
  title={Is There a Better Source Distribution than Gaussian? Exploring Source Distributions for Image Flow Matching},
  author={Lee, Junho and Kim, Kwanseok and Lee, Joonseok},
  journal=TMLR,
  year={2025}
}

@article{babaeizadeh2021fitvid,
  title={Fitvid: Overfitting in pixel-level video prediction},
  author={Babaeizadeh, Mohammad and Saffar, Mohammad Taghi and Nair, Suraj and Levine, Sergey and Finn, Chelsea and Erhan, Dumitru},
  journal={arXiv preprint arXiv:2106.13195},
  year={2021}
}

@article{graves2012long,
  title={Long short-term memory},
  author={Graves, Alex},
  journal={Supervised sequence labelling with recurrent neural networks},
  year={2012},
}

@inproceedings{vaswani2017attention,
  title={Attention is all you need},
  author={Vaswani, Ashish and Shazeer, Noam and Parmar, Niki and Uszkoreit, Jakob and Jones, Llion and Gomez, Aidan N and Kaiser, {\L}ukasz and Polosukhin, Illia},
  booktitle=NIPS,
  year={2017}
}

@inproceedings{le2021ccvs,
  title={Ccvs: Context-aware controllable video synthesis},
  author={Le Moing, Guillaume and Ponce, Jean and Schmid, Cordelia},
  booktitle=NeurIPS,
  year={2021}
}

@article{ma2024latte,
  title={Latte: Latent diffusion transformer for video generation},
  author={Ma, Xin and Wang, Yaohui and Chen, Xinyuan and Jia, Gengyun and Liu, Ziwei and Li, Yuan-Fang and Chen, Cunjian and Qiao, Yu},
  journal=TMLR,
  year={2025}
}

@inproceedings{dhariwal2021diffusion,
  title={Diffusion models beat gans on image synthesis},
  author={Dhariwal, Prafulla and Nichol, Alexander},
  booktitle=NeurIPS,
  year={2021}
}

@inproceedings{song2020score,
  title={Score-based generative modeling through stochastic differential equations},
  author={Song, Yang and Sohl-Dickstein, Jascha and Kingma, Diederik P and Kumar, Abhishek and Ermon, Stefano and Poole, Ben},
  booktitle=ICLR,
  year={2021}
}

@inproceedings{janner2022planning,
  title={Planning with Diffusion for Flexible Behavior Synthesis},
  author={Janner, Michael and Du, Yilun and Tenenbaum, Joshua and Levine, Sergey},
  booktitle=ICML,
  year={2022},
}

@inproceedings{carvalho2023motion,
  title={Motion planning diffusion: Learning and planning of robot motions with diffusion models},
  author={Carvalho, Joao and Le, An T and Baierl, Mark and Koert, Dorothea and Peters, Jan},
  booktitle={2023 IEEE/RSJ International Conference on Intelligent Robots and Systems (IROS)},
  year={2023},
}

@article{carvalho2025motion,
  title={Motion planning diffusion: Learning and adapting robot motion planning with diffusion models},
  author={Carvalho, Joao and Le, An T and Kicki, Piotr and Koert, Dorothea and Peters, Jan},
  journal={IEEE Transactions on Robotics},
  year={2025},
}

@inproceedings{ge2023preserve,
  title={Preserve your own correlation: A noise prior for video diffusion models},
  author={Ge, Songwei and Nah, Seungjun and Liu, Guilin and Poon, Tyler and Tao, Andrew and Catanzaro, Bryan and Jacobs, David and Huang, Jia-Bin and Liu, Ming-Yu and Balaji, Yogesh},
  booktitle=ICCV,
  year={2023}
}

@inproceedings{luo2023videofusion,
  title={Videofusion: Decomposed diffusion models for high-quality video generation},
  author={Luo, Zhengxiong and Chen, Dayou and Zhang, Yingya and Huang, Yan and Wang, Liang and Shen, Yujun and Zhao, Deli and Zhou, Jingren and Tan, Tieniu},
  booktitle=CVPR,
  year={2023}
}

@inproceedings{zhang2024trip,
  title={Trip: Temporal residual learning with image noise prior for image-to-video diffusion models},
  author={Zhang, Zhongwei and Long, Fuchen and Pan, Yingwei and Qiu, Zhaofan and Yao, Ting and Cao, Yang and Mei, Tao},
  booktitle=CVPR,
  year={2024}
}

@inproceedings{meng2021sdedit,
  title={Sdedit: Guided image synthesis and editing with stochastic differential equations},
  author={Meng, Chenlin and He, Yutong and Song, Yang and Song, Jiaming and Wu, Jiajun and Zhu, Jun-Yan and Ermon, Stefano},
  booktitle=ICLR,
  year={2022}
}

@inproceedings{peebles2023scalable,
  title={Scalable diffusion models with transformers},
  author={Peebles, William and Xie, Saining},
  booktitle=ICCV,
  year={2023}
}

@inproceedings{sadat2023cads,
  title={CADS: Unleashing the diversity of diffusion models through condition-annealed sampling},
  author={Sadat, Seyedmorteza and Buhmann, Jakob and Bradley, Derek and Hilliges, Otmar and Weber, Romann M},
  booktitle=ICLR,
  year={2024}
}

@inproceedings{chen2025car,
  title={CAR-Flow: Condition-Aware Reparameterization Aligns Source and Target for Better Flow Matching},
  author={Chen, Chen and Guo, Pengsheng and Song, Liangchen and Lu, Jiasen and Qian, Rui and Wang, Xinze and Fu, Tsu-Jui and Liu, Wei and Yang, Yinfei and Schwing, Alex},
  booktitle=NeurIPS,
  year={2025}
}

@inproceedings{reda2022film,
  title={Film: Frame interpolation for large motion},
  author={Reda, Fitsum and Kontkanen, Janne and Tabellion, Eric and Sun, Deqing and Pantofaru, Caroline and Curless, Brian},
  booktitle=ECCV,
  year={2022},
}

@inproceedings{lai2018learning,
  title={Learning blind video temporal consistency},
  author={Lai, Wei-Sheng and Huang, Jia-Bin and Wang, Oliver and Shechtman, Eli and Yumer, Ersin and Yang, Ming-Hsuan},
  booktitle=ECCV,
  year={2018}
}

@article{yang2023diffusion,
  title={Diffusion probabilistic modeling for video generation},
  author={Yang, Ruihan and Srivastava, Prakhar and Mandt, Stephan},
  journal={Entropy},
  year={2023},
}

@article{xue2026acwmphys,
  title={ACWM-Phys: Investigating Generalized Physical Interaction in Action-Conditioned Video World Models},
  author={Xue, Haotian and Chen, Yipu and Ma, Liqian and Zhao, Zelin and Moukheiber, Lama and Zhu, Yuchen and Chen, Yongxin},
  journal={arXiv preprint arXiv:2605.08567},
  year={2026}
}
}

\restoreappendixlayout
\appendix
\newpage
\begin{center}
    \LARGE{\textbf{Appendix}}
\end{center}

\section{Algorithms of \method{}} \label{app:algorithms}
The training procedures of CNP-Flow are summarized in Algo.~\ref{alg:train_cnp} and Algo.~\ref{alg:train_cfm}. The inference procedures for video prediction are provided in Algo.~\ref{alg:vp}, while the inference procedures for video interpolation are provided in Algo.~\ref{alg:vi}.

%

\begingroup
\makeatletter
\algrenewcommand\ALG@beginalgorithmic{\footnotesize}
\makeatother
\setlength{\textfloatsep}{2pt}
\setlength{\floatsep}{2pt}
\setlength{\intextsep}{2pt}
\setlength{\abovecaptionskip}{2pt}
\setlength{\belowcaptionskip}{2pt}

\noindent
\begin{minipage}[t]{0.495\linewidth}
\begin{algorithm}[H]
\caption{Training the CNP}
\label{alg:train_cnp}
\begin{algorithmic}[1]
\Require Temporal dataset $\mathcal{D}_\mathcal{T}$, pretrained FM $\theta$, KL weight $\lambda$, learning rate $\gamma$
\Repeat
  \State Sample $(\bm{x}_1^{i}, \mathcal{C}_\mathcal{T}^i) \sim \mathcal{D}_\mathcal{T}$
  \State Flow reversal: $\Tilde{\bm{x}}_0^i = \widetilde{\Phi}_{\theta, 1 \to 0}(\bm{x}_1^i \mid \mathcal{C}_\mathcal{T}^i)$
  \State $\mathcal{L}_\text{NLL}(\phi) = -\mathbb{E}[\log p_\phi(\Tilde{\bm{x}}_0^i \mid \mathcal{C}_\mathcal{T}^i)]$
  \State $\mathcal{L}_\text{KL}(\phi) = \mathcal{D}_\text{KL}\left(p_\phi(\cdot \mid \mathcal{C}_\mathcal{T}^i) \| \mathcal{N}(\mathbf{0},\mathbf{I})\right)$
  \State $\mathcal{L}_{\mathrm{CNP}} = \mathcal{L}_{\mathrm{NLL}} + \lambda \mathcal{L}_{\mathrm{KL}}$
  \State $\phi \leftarrow \phi - \gamma \nabla_\phi \mathcal{L}_{\mathrm{CNP}}$
\Until{convergence}
\Ensure Optimized CNP $\phi$
\end{algorithmic}
\end{algorithm}
\end{minipage}\hfill
\begin{minipage}[t]{0.495\linewidth}
\begin{algorithm}[H]
\caption{Fine-Tuning FM}
\label{alg:train_cfm}
\begin{algorithmic}[1]
\Require Temporal dataset $\mathcal{D}_\mathcal{T}$, CNP $\phi$, FM $\theta$, learning rate $\gamma$
\Repeat
  \State Sample $(\bm{x}_1^{i}, \mathcal{C}_\mathcal{T}^i) \sim \mathcal{D}_\mathcal{T}$
  \State $\bm{\epsilon} \sim \mathcal{N}(\mathbf{0},\mathbf{I})$;\ \ $\bm{x}_0^i = \bm{\mu}_\phi(\mathcal{C}_\mathcal{T}^i) + \sigma_\phi(\mathcal{C}_\mathcal{T}^i)\,\bm{\epsilon}$
  \State $t \sim \mathcal{U}(0, 1)$; \ \ $\bm{x}_t^i = (1-t)\bm{x}_0^i + t\bm{x}_1^i$
  \State $\bm{v}_t^i = \bm{x}_1^i - \bm{x}_0^i$
  \State $\mathcal{L}_\text{FT-FM}(\theta) = \mathbb{E}\!\left[\| \bm{v}_\theta(\bm{x}_t^i, t \mid \mathcal{C}_\mathcal{T}^i) - \bm{v}_t^i \|_2^2 \right]$
  \State $\theta \leftarrow \theta - \gamma \nabla_\theta \mathcal{L}_{\text{FT-FM}}$
\Until{convergence}
\Ensure Optimized FM $\theta$
\end{algorithmic}
\end{algorithm}
\end{minipage}

\vspace{2pt}

\noindent
\begin{minipage}[t]{0.495\linewidth}
\begin{algorithm}[H]
\caption{Video Prediction with \method{}}
\label{alg:vp}
\begin{algorithmic}[1]
\Require Initial $\bm{x}^0,\bm{x}^1$, CNP $\phi$, FM $\theta$, timesteps $0=t_0<\cdots<t_K=1$
\State $\bm{x}_1^0 \leftarrow \bm{x}^0,\ \bm{x}_1^1 \leftarrow \bm{x}^1$
\For{$i=2$ \textbf{to} $N$}
  \State VP condition: $\mathcal{C}_\mathcal{T}^i=(\bm{x}_1^{i-1}, \bm{x}_1^{c}, i-c)$
  \State $\bm{\epsilon}\sim\mathcal{N}(\mathbf{0},\mathbf{I})$;\ \ $\bm{x}_0^i = \bm{\mu}_\phi(\mathcal{C}_\mathcal{T}^i) + \sigma_\phi(\mathcal{C}_\mathcal{T}^i)\,\bm{\epsilon}$
  \For{$k=0$ \textbf{to} $K-1$}
    \State $\Delta t_k = t_{k+1} - t_k$
    \State $\bm{x}_{t_{k+1}}^i = \bm{x}_{t_k}^i + \bm{v}_\theta(\bm{x}_{t_k}^i, t_k \mid \mathcal{C}_\mathcal{T}^i) \Delta t_k$
  \EndFor
  \State $\bm{x}^i \leftarrow \bm{x}_1^i$
\EndFor
\Ensure Temporal predictions $\{\bm{x}^i\}_{i=0}^{N}$
\end{algorithmic}
\end{algorithm}
\end{minipage}\hfill
\begin{minipage}[t]{0.495\linewidth}
\begin{algorithm}[H]
\caption{Video Interpolation with \method{}}
\label{alg:vi}
\begin{algorithmic}[1]
\Require Initial $\bm{x}^0,\bm{x}^N$, CNP $\phi$, FM $\theta$, timesteps $0=t_0<\cdots<t_K=1$
\State $\bm{x}_1^0 \leftarrow \bm{x}^0,\ \bm{x}_1^N \leftarrow \bm{x}^N$
\For{$i=1$ \textbf{to} $N-1$}
  \State VI condition: $\mathcal{C}_\mathcal{T}^i=(\bm{x}_1^{i-1}, \bm{x}_1^{N}, N-i)$ 
  \State $\bm{\epsilon}\sim\mathcal{N}(\mathbf{0},\mathbf{I})$;\ \ $\bm{x}_0^i = \bm{\mu}_\phi(\mathcal{C}_\mathcal{T}^i) + \sigma_\phi(\mathcal{C}_\mathcal{T}^i)\,\bm{\epsilon}$
  \For{$k=0$ \textbf{to} $K-1$}
    \State $\Delta t_k = t_{k+1} - t_k$
    \State $\bm{x}_{t_{k+1}}^i = \bm{x}_{t_k}^i + \bm{v}_\theta(\bm{x}_{t_k}^i, t_k \mid \mathcal{C}_\mathcal{T}^i) \Delta t_k$
  \EndFor
  \State $\bm{x}^i \leftarrow \bm{x}_1^i$
\EndFor
\Ensure Temporal predictions $\{\bm{x}^i\}_{i=0}^{N}$
\end{algorithmic}
\end{algorithm}
\end{minipage}
\endgroup

\section{Mathematical Details}
\label{app:mathematical_details}
We consider the general form of the CNP objective and derive the equations satisfied by the corresponding optimal CNP parameters.
\begin{proposition}[Population minimizer of CNP objective]
\label{prop:cnp_min}
Consider the CNP objective $\mathcal{L}_{\mathrm{CNP}}(\phi) =
\alpha \mathcal{L}_{\mathrm{NLL}}(\phi) + \beta \mathcal{L}_{\mathrm{KL}}(\phi)$,
with $\alpha > 0, \beta \geq 0$. Let the CNP family be the isotropic Gaussians $p_\phi(\bm{x}_0|\mathcal{C}_\mathcal{T}) = \mathcal{N}(\bm{x}_0; \bm{\mu}_\phi(\mathcal{C}_\mathcal{T}), \sigma_\phi^2(\mathcal{C}_\mathcal{T}) \mathbf{I})$.
For each condition $\mathcal{C}_\mathcal{T}$, assume the flow-reversal-induced source distribution has finite second moments, and denote its mean and variance by
\begin{equation*}
     \Tilde{\bm{\mu}}_0(\mathcal{C}_\mathcal{T}) = \mathbb E_{\bm{x}_1 \sim p_1(\cdot \mid \mathcal{C_T})}[\Tilde{\Phi}_{\theta, 1 \to 0}(\bm{x}_1 \mid \mathcal{C_T})], \quad
     \Tilde{\sigma}_0^2(\mathcal{C}_\mathcal{T}) = \frac{1}{d} \mathbb E_{\bm{x}_1 \sim p_1(\cdot \mid \mathcal{C_T})}
    \|\Tilde{\Phi}_{\theta, 1 \to 0}(\bm{x}_1\mid\mathcal{C_T}) - \Tilde{\bm{\mu}}_0\|_2^2.
\end{equation*}

Then for each $\mathcal{C}_\mathcal{T}$, the population minimizer of $\mathcal{L}_{\mathrm{CNP}}$ over the isotropic Gaussian family is
\begin{align}
p_\phi^*(\bm{x}_0 \mid \mathcal{C}_\mathcal{T})
= \mathcal{N}(\bm{x}_0; \bm{\mu}_\phi^*(\mathcal{C}_\mathcal{T}), {\sigma_\phi^*}^2(\mathcal{C}_\mathcal{T}) \mathbf{I}),
\end{align}
whose parameters $(\bm{\mu}_\phi^*, {\sigma_\phi^*}^2)$ satisfy
\begin{equation}
\begin{aligned}
    &\bm{\mu}_\phi^* =\frac{\alpha}{\alpha+\beta {{\sigma}_\phi^*}^2}\Tilde{\bm{\mu}}_0, \text{ and } \\
    &\beta ({\sigma_\phi^*}^2)^2 + (\alpha - \beta) {\sigma_\phi^*}^2 - \alpha \Tilde{\sigma}_0^2 -\frac{\alpha\beta^2({\sigma_\phi^*}^2)^2}{d(\alpha+\beta {\sigma_\phi^*}^2)^2}\|\Tilde{\bm{\mu}}_0\|_2^2 = 0.
\end{aligned}
\end{equation}

\end{proposition}

\begin{proof}
Given a condition $\mathcal{C}_\mathcal{T}$, let the flow-reversal-induced source distribution be $p_0(\cdot \mid \mathcal{C}_\mathcal{T}) = (\Tilde{\Phi}_{1 \to 0}(\cdot \mid \mathcal{C}_\mathcal{T}))_\#p_{\mathcal{D}_\mathcal{T}}$. The negative log-likelihood (NLL) of $p_\phi(\bm{x}_0 \mid \mathcal{C}_\mathcal{T})$ is
\begin{align*}
\mathcal{L}_{\mathrm{NLL}}(\phi) =
-\mathbb{E}_{\bm{x}_0 \sim p_0}
\log \mathcal{N}(\bm{x}_0; \bm{\mu}_\phi, \sigma_\phi^2 \mathbf{I})
= \frac{d}{2} \log (2 \pi \sigma_\phi^2)
+ \frac{1}{2\sigma_\phi^2}
\mathbb{E}_{\bm{x}_0 \sim p_0}
\|\bm{x}_0 - \bm{\mu}_\phi\|_2^2.
\end{align*}

By the second-moment decomposition $\mathbb E_{\bm{x}_0 \sim p_0} \|\bm{x}_0-\bm{\mu}_\phi\|_2^2
= d\Tilde{\sigma}_0^2 + \|\bm{\mu}_\phi - \Tilde{\bm{\mu}}_0\|_2^2$,
we obtain
\begin{align*}
\mathcal L_{\mathrm{NLL}}(\phi) =
\frac{d}{2} \log(2\pi \sigma_\phi^2) + \frac{1}{2\sigma_\phi^2} \left(d \Tilde{\sigma}_0^2 + \|\bm{\mu}_\phi -\Tilde{\bm{\mu}}_0\|_2^2 \right). 
\end{align*}

Moreover, the KL term against the standard Gaussian has the closed form
\begin{align*}
\mathcal L_{\mathrm{KL}}(\phi)
=\frac{1}{2}
\left(
\|\bm{\mu}_\phi\|_2^2 + d \sigma_\phi^2 -d -d \log \sigma_\phi^2
\right).
\end{align*}

Therefore, the CNP objective becomes
\begin{align*}
\mathcal L_{\mathrm{CNP}}(\phi)
= \frac{\alpha}{2\sigma_\phi^2}
\|\bm{\mu}_\phi-\Tilde{\bm{\mu}}_0\|_2^2
+ \frac{\beta}{2}\|\bm{\mu}_\phi\|_2^2
+ \frac{\alpha d}{2}\log\sigma_\phi^2
+ \frac{\alpha d\Tilde{\sigma}_0^2}{2\sigma_\phi^2}
+ \frac{\beta d}{2}\sigma_\phi^2
- \frac{\beta d}{2}\log\sigma_\phi^2
+ \frac{\alpha d}{2} \log 2\pi
- \frac{\beta d}{2}.
\end{align*}

At a stationary point $(\bm{\mu}_\phi^*, {\sigma_\phi^*}^2)$, the first-order optimality conditions satisfy
\begin{align*}
\nabla_\phi \mathcal L_{\mathrm{CNP}}
= \left[
\begin{array}{c}
    \nabla_{\bm{\mu}_\phi}\mathcal{L}_{\mathrm{CNP}}^\top(\bm{\mu}_\phi^*, {\sigma_\phi^*}^2) \quad
    \dfrac{\partial \mathcal{L}_{\mathrm{CNP}}}{\partial \sigma_{\phi}^2}(\bm{\mu}_\phi^*, {\sigma_\phi^*}^2)
\end{array}
\right]^\top =
\left[
\begin{array}{c}
    \mathbf{0}_d^\top \quad 0
\end{array}
\right]^\top.
\end{align*}

Differentiating $\mathcal{L}_{\mathrm{CNP}}$ with respect to $\bm{\mu}_\phi$ yields
\begin{align*}
\nabla_{\bm{\mu}_\phi}\mathcal{L}_{\mathrm{CNP}}(\bm{\mu}_\phi^*, {\sigma_\phi^*}^2)
= \frac{\alpha}{\sigma_\phi^2}(\bm{\mu}_\phi^* - \Tilde{\bm{\mu}}_0) + \beta \bm{\mu}_\phi^*.
\end{align*}

Setting this gradient to zero gives
\begin{align*}
\bm{\mu}_\phi^* = \frac{\alpha}{\alpha+\beta {\sigma_\phi^*}^2}\Tilde{\bm{\mu}}_0.
\end{align*}

Substituting the optimal mean $\bm{\mu}_\phi^*$ into $\frac{\partial \mathcal{L}_{\mathrm{CNP}}}{\partial \sigma_{\phi}^2}(\bm{\mu}_\phi^*, {\sigma_\phi^*}^2) = 0$, we obtain
\begin{align*}
\beta ({\sigma_\phi^*}^2)^2 + (\alpha - \beta) {\sigma_\phi^*}^2 - \alpha \Tilde{\sigma}_0^2 -\frac{\alpha\beta^2({\sigma_\phi^*}^2)^2}{d(\alpha+\beta {\sigma_\phi^*}^2)^2}\|\Tilde{\bm{\mu}}_0\|_2^2 = 0.
\end{align*}
The proof is complete.
\end{proof}

Next, we consider the limiting behavior of the optimal parameters when the two loss terms dominate under extreme coefficient settings.
\begin{corollary}
\label{cor:cnp_limits}
Under the setting of Prop.~\ref{prop:cnp_min}, let
$p_\phi^*(\bm{x}_0\mid\mathcal{C}_\mathcal{T})
= \mathcal{N}(\bm{x}_0; \bm{\mu}_\phi^*(\mathcal{C}_\mathcal{T}), {\sigma_\phi^*}^2(\mathcal{C}_\mathcal{T}) \mathbf{I})$
be the population minimizer of $\mathcal{L}_{\mathrm{CNP}}$ over the isotropic Gaussian family. Then
\begin{enumerate}
\item If $\beta / \alpha \to 0^+$, then the population minimizer reduces to the moment-matching optimum,
\begin{equation*}
    \bm{\mu}_\phi^* \to \Tilde{\bm{\mu}}_0,
    \text{  and  }
    {\sigma_\phi^*}^2 \to \Tilde{\sigma}_0^2.
\end{equation*}
\item If $\alpha / \beta \to 0^+$, then the population minimizer reduces to the standard Gaussian distribution,
\begin{equation*}
\bm{\mu}_\phi^* \to \mathbf{0}_d,
\text{  and  }
{\sigma_\phi^*}^2 \to 1.
\end{equation*}
\end{enumerate}
\end{corollary}

\begin{proof}
For the first limiting regime, let $\eta = \beta / \alpha$. Then
\begin{equation*}
\bm{\mu}_\phi^* = \frac{1}{1 + \eta {\sigma_\phi^*}^2}\Tilde{\bm{\mu}}_0,
\end{equation*}
and the variance equation becomes
\begin{equation*}
\eta ({\sigma_\phi^*}^2)^2 + (1 - \eta) {\sigma_\phi^*}^2 - \Tilde{\sigma}_0^2 -\frac{\eta^2{(\sigma_\phi^*}^2)^2}{d(1  + \eta {\sigma_\phi^*}^2)^2}\|\Tilde{\bm{\mu}}_0\|_2^2 = 0.
\end{equation*}

Letting $\eta \to 0^+$, the nonlinear correction term vanishes, and the equation reduces to ${\sigma_\phi^*}^2 - \Tilde{\sigma}_0^2 = 0$. Hence ${\sigma_\phi^*}^2 \to \Tilde{\sigma}_0^2$ and consequently $\bm{\mu}_\phi^* \to \Tilde{\bm{\mu}}_0$ .

For the second limiting regime, let $\tau = \alpha / \beta$. Then
\begin{align*}
\bm{\mu}_\phi^* = \frac{\tau}{\tau + {\sigma_\phi^*}^2}\Tilde{\bm{\mu}}_0,
\end{align*}
and the variance equation becomes
\begin{align*}
({\sigma_\phi^*}^2)^2 + (\tau-1) {\sigma_\phi^*}^2 - \tau \Tilde{\sigma}_0^2 -\frac{\tau {(\sigma_\phi^*}^2)^2}{d(\tau + {\sigma_\phi^*}^2)^2}\|\Tilde{\bm{\mu}}_0\|_2^2 = 0.
\end{align*}

Letting $\tau \to 0^+$, the equation reduces to
$({\sigma_\phi^*}^2)^2 - {\sigma_\phi^*}^2 = 0$. Since ${\sigma_\phi^*}^2 > 0$, thus ${\sigma_\phi^*}^2 \to 1$ and $\bm{\mu}_\phi^* \to \mathbf{0}_d$. 
The proof is complete.
\end{proof}

Prop.~\ref{prop:cnp_min} shows that, after introducing the KL regularization term, the optimal mean and standard deviation become coupled and jointly influence each other. In particular, increasing the KL regularization strength pulls the mean toward zero while simultaneously constraining the variance toward the standard Gaussian prior. Cor.~\ref{cor:cnp_limits} further analyzes the limiting regimes of this objective. When the KL coefficient dominates the NLL term, the learned distribution converges toward $\mathcal{N}(\mathbf{0}, \mathbf{I})$; conversely, when the KL regularization becomes negligible, the solution reduces to the setting characterized in Prop.~\ref{prop:cpn_min_nll} in the main paper.

\begin{lemma}[Lipschitz continuity of the conditional flow reversal map]
Suppose that $\bm{v}_\theta(\bm{x}_t, t \mid \mathcal{C}_\mathcal{T})$ is Lipschitz continuous in ($\bm{x}_t$, $\mathcal{C}_\mathcal{T}$) on $\mathcal{X}_t \times \mathcal{X}_{\mathcal{C}_\mathcal{T}}$, uniformly over $t \in [0, 1]$. Then $\Tilde{\Phi}_{\theta, 1 \to 0}(\bm{x}_1 \mid \mathcal{C}_\mathcal{T})$ is Lipschitz continuous on $\mathcal{X}_1 \times \mathcal{X}_{\mathcal{C}_\mathcal{T}}$.
\end{lemma}

\begin{proof}
For any $\bm{x}_1, \bm{x}_1' \in \mathcal{X}_1$ and $\mathcal{C}_\mathcal{T}, \mathcal{C}_\mathcal{T}' \in \mathcal{X}_{\mathcal{C}_\mathcal{T}}$, since $\bm{v}_\theta(\bm{x}_t, t \mid \mathcal{C}_\mathcal{T})$ is Lipschitz continuous on $\mathcal{X}_t \times \mathcal{X}_{\mathcal{C}_\mathcal{T}}$, there exist $L_{\bm{x}}, L_{\mathcal{C}} >0$, independent of $t$, such that for all $t \in [0, 1]$,
\begin{align*}
\|\bm{v}_\theta(\bm{x}_t, t \mid \mathcal{C}_\mathcal{T}) -\bm{v}_\theta(\bm{x}_t', t \mid \mathcal{C}_\mathcal{T}')\|_2
\leq L_{\bm{x}} \|\bm{x}_t - \bm{x}_t'\|_2
+ L_{\mathcal{C}} \|\mathcal{C}_\mathcal{T} - \mathcal{C}_\mathcal{T}'\|_2.
\end{align*}
Define the reverse-time trajectories $\bm{z}_\tau := \bm{x}_{1-\tau}$ and $\bm{z}_\tau' := \bm{x}_{1-\tau}'$ for $\tau \in [0, 1]$.
By definition,
\begin{align*}
\bm{z}_1 = \bm{x}_0 = \Tilde \Phi_{\theta, 1 \to 0}(\bm{x}_1 \mid \mathcal{C}_\mathcal{T})
\quad \text{and} \quad
\bm{z}_1' = \bm{x}_0' = \Tilde \Phi_{\theta, 1 \to 0}(\bm{x}_1' \mid \mathcal{C}_\mathcal{T}'),
\end{align*}
and the reverse ODEs are given by
\begin{align*}
\frac{\mathrm{d}\bm{z}_\tau}{d\tau}
= -\bm{v}_\theta(\bm{z}_\tau, 1-\tau \mid \mathcal{C}_\mathcal{T})
\quad \text{and} \quad
\frac{\mathrm{d}\bm{z}_\tau'}{d\tau} = -\bm{v}_\theta(\bm{z}_\tau', 1-\tau \mid \mathcal{C}_\mathcal{T}').
\end{align*}
It follows that
\begin{align*}
\frac{\mathrm{d}}{\mathrm{d}\tau}\|\bm{z}_\tau - \bm{z}_\tau'\|_2
&\leq
\left\|
\frac{\mathrm{d}\bm{z}_\tau}{\mathrm{d}\tau} - \frac{\mathrm{d}\bm{z}_\tau'}{\mathrm{d}\tau}
\right\|_2 \\
&=
\left\|
-\bm{v}_\theta(\bm{z}_\tau, 1-\tau \mid \mathcal{C}_\mathcal{T}) -
[-\bm{v}_\theta(\bm{z}_\tau', 1-\tau \mid \mathcal{C}_\mathcal{T}')]
\right\|_2 \\
&\leq
L_{\bm{x}} \|\bm{z}_\tau - \bm{z}_\tau'\|_2
+ L_{\mathcal{C}} \|\mathcal{C}_\mathcal{T} - \mathcal{C}_\mathcal{T}'\|_2.
\end{align*}
By Grönwall's inequality, we have
\begin{align*}
\|\bm{z}_1 - \bm{z}_1'\|_2
&\leq \mathrm{e}^{L_{\bm{x}}} \|\bm{z}_0 - \bm{z}_0'\|_2
+ \int_0^1 \mathrm{e}^{L_{\bm{x}} (1-\tau)} L_{\mathcal{C}} \| \mathcal{C}_\mathcal{T} -\mathcal{C}_\mathcal{T}'\|_2 \mathrm{d}\tau \\
&= \mathrm{e}^{L_{\bm{x}}} \|\bm{z}_0 - \bm{z}_0'\|_2
+ L_{\mathcal{C}} \| \mathcal{C}_\mathcal{T} -\mathcal{C}_\mathcal{T}'\|_2 \int_0^1 \mathrm{e}^{L_{\bm{x}} (1-\tau)} \mathrm{d}\tau \\
&= \mathrm{e}^{L_{\bm{x}}} \|\bm{z}_0 - \bm{z}_0'\|_2
+ L_{\mathcal{C}} \| \mathcal{C}_\mathcal{T} -\mathcal{C}_\mathcal{T}'\|_2 \left( \frac{e^{L_{\bm{x}}}-1}{L_{\bm{x}}} \right) \\
&= \mathrm{e}^{L_{\bm{x}}} \|\bm{z}_0 - \bm{z}_0'\|_2
+ \frac{L_{\mathcal{C}}}{L_{\bm{x}}} (e^{L_{\bm{x}}}-1) \| \mathcal{C}_\mathcal{T} -\mathcal{C}_\mathcal{T}'\|_2.
\end{align*}

Since $\bm{z}_0 = \bm{x}_1$, $\bm{z}_0' = \bm{x}_1'$, $\bm{z}_1 = \tilde{\Phi}_{\theta, 1\to 0}(\bm{x}_1 \mid \mathcal{C}_\mathcal{T})$
and $\bm{z}_1' = \tilde{\Phi}_{\theta, 1\to 0}(\bm{x}_1' \mid \mathcal{C}_\mathcal{T}')$, we obtain
\begin{align*}
\left\|
\Tilde \Phi_{\theta, 1 \to 0}(\bm{x}_1 \mid \mathcal{C}_\mathcal{T}) -
\Tilde \Phi_{\theta, 1 \to 0}(\bm{x}_1' \mid \mathcal{C}_\mathcal{T}')
\right\|_2
&\leq
L_\Phi \|\bm{x}_1 - \bm{x}_1'\|_2
+ L_\Phi \|\mathcal{C}_\mathcal{T}-\mathcal{C}_\mathcal{T}'\|_2,
\end{align*}
where $L_\Phi = \max\{\mathrm{e}^{L_{\bm{x}}}, \frac{L_{\mathcal{C}}}{L_{\bm{x}}}(\mathrm{e}^{L_{\bm{x}}} - 1)\}$.
\end{proof}

\begin{proposition}[Lipschitz continuity of the NLL objective population minimizer]
Suppose that $\Tilde{\Phi}_{\theta, 1 \to 0}(\bm{x}_1 \mid \mathcal{C}_\mathcal{T})$ is Lipschitz continuous in $(\bm{x}_1, \mathcal{C}_\mathcal{T})$ on $\mathcal{X}_1 \times \mathcal{X}_{\mathcal{C}_\mathcal{T}}$, $\mathcal C_{\mathcal T}\mapsto p_1(\cdot\mid \mathcal C_{\mathcal T})$ is Lipschitz continuous on $\mathcal{X}_{\mathcal{C}_\mathcal{T}}$ with respect to the $W_2$-distance, and there exists $M>0$ such that $\mathbb{E}_{\bm{x}_1\sim p_1(\cdot\mid \mathcal C_{\mathcal T})} \|\Tilde{\Phi}_{\theta,1\to0}(\bm{x}_1\mid \mathcal C_{\mathcal T}) \|_2^2 \le M$ for all $\mathcal C_{\mathcal T}\in \mathcal X_{\mathcal C_{\mathcal T}}$. Then $\Tilde{\bm{\mu}}_0(\mathcal{C}_{\mathcal{T}})$ and ${\Tilde{\sigma}_0}^2(\mathcal{C}_{\mathcal{T}})$ are Lipschitz continuous on $\mathcal{X}_{\mathcal{C}_\mathcal{T}}$.
\end{proposition}

\begin{proof} 
\paratitle{Step 1: Lipschitz continuity of the optimal mean.}
For any $\mathcal{C}_\mathcal{T}', \mathcal{C}_\mathcal{T} \in \mathcal{X}_{\mathcal{C}_\mathcal{T}}$, since $\Tilde{\Phi}_{\theta, 1 \to 0}(\bm{x}_1 \mid \mathcal{C}_\mathcal{T})$ is Lipschitz continuous on $\mathcal{X}_1 \times \mathcal{X}_{\mathcal{C}_\mathcal{T}}$, there exists $L_\Phi >0$ such that
\begin{align*}
    \left\|\Tilde{\Phi}_{\theta, 1 \to 0}(\bm{x}_1' \mid \mathcal{C}_{\mathcal{T}}') - \Tilde{\Phi}_{\theta, 1 \to 0}(\bm{x}_1 \mid \mathcal{C}_\mathcal{T}) \right\|_2
    \leq L_\Phi \left(\| \bm{x}_1' - \bm{x}_1\|_2 + \| \mathcal{C}_\mathcal{T}' - \mathcal{C}_\mathcal{T}\|_2 \right) \text{ for all $\bm{x}_1', \bm{x}_1 \in \mathcal{X}_1$}.
\end{align*}
Moreover, since $\mathcal C_{\mathcal T}\mapsto p_1(\cdot\mid \mathcal C_{\mathcal T})$ is Lipschitz continuous on $\mathcal{X}_{\mathcal{C}_\mathcal{T}}$ with respect to the $W_2$-distance, there exists $L_{p_1} >0$ such that
\begin{equation*}
    W_2(p_1(\cdot \mid \mathcal{C}_{\mathcal{T}}'), p_1(\cdot \mid \mathcal{C}_{\mathcal{T}})) \leq L_{p_1}
    \| \mathcal{C}_\mathcal{T}' - \mathcal{C}_\mathcal{T}\|_2.
\end{equation*}
Furthermore, there exists $\gamma_1 \in \Pi(p_1(\cdot \mid \mathcal{C}_\mathcal{T}'), p_1(\cdot \mid \mathcal{C}_{\mathcal{T}}))$ such that
\begin{align*}
    \mathbb{E}_{(\bm{x}_1', \bm{x}_1) \sim \gamma_1} \|\bm{x}_1' - \bm{x}_1\|_2^2 \leq \inf_{\gamma^* \in \Pi(p_1(\cdot \mid \mathcal{C}_\mathcal{T}'), p_1(\cdot \mid \mathcal{C}_{\mathcal{T}}))}\mathbb{E}_{(\bm{x}_1', \bm{x}_1) \sim \gamma^*} \|\bm{x}_1' - \bm{x}_1\|_2^2 + \| \mathcal{C}_\mathcal{T}' - \mathcal{C}_\mathcal{T}\|_2^2.
\end{align*}
Then we obtain
\begin{align*}
    \mathbb{E}_{(\bm{x}_1', \bm{x}_1) \sim \gamma_1} \|\bm{x}_1' - \bm{x}_1\|_2
    &\leq \left( \mathbb{E}_{(\bm{x}_1', \bm{x}_1) \sim \gamma_1} \|\bm{x}_1' - \bm{x}_1\|_2^2 \right)^\frac{1}{2} \text{ (by Jensen's inequality)} \\
    &\leq W_2(p_1(\cdot \mid \mathcal{C}_\mathcal{T}'), p_1(\cdot \mid \mathcal{C}_{\mathcal{T}})) + \| \mathcal{C}_\mathcal{T}' - \mathcal{C}_\mathcal{T}\|_2
    \leq (L_{p_1} + 1) \| \mathcal{C}_\mathcal{T}' - \mathcal{C}_\mathcal{T}\|_2.
\end{align*}
Therefore, we have 
\begin{align*}
\left\| \Tilde{\bm{\mu}}_0(\mathcal{C}_\mathcal{T}') - \Tilde{\bm{\mu}}_0(\mathcal{C}_{\mathcal{T}}) \right\|_2
&= \left\|
\mathbb E_{\bm{x}_1' \sim p_1(\cdot \mid \mathcal{C}_\mathcal{T}')}[\Tilde{\Phi}_{\theta, 1 \to 0}(\bm{x}_1' \mid \mathcal{C}_\mathcal{T}')]
- \mathbb E_{\bm{x}_1 \sim p_1(\cdot \mid \mathcal{C}_\mathcal{T})}[\Tilde{\Phi}_{\theta, 1 \to 0}(\bm{x}_1 \mid \mathcal{C}_\mathcal{T})] \right\|_2 \\
&= \left\|
\mathbb E_{(\bm{x}_1', \bm{x}_1) \sim \gamma_1}[\Tilde{\Phi}_{\theta, 1 \to 0}(\bm{x}_1' \mid \mathcal{C}_\mathcal{T}')]
- \mathbb E_{(\bm{x}_1', \bm{x}_1) \sim \gamma_1}[\Tilde{\Phi}_{\theta, 1 \to 0}(\bm{x}_1 \mid \mathcal{C}_\mathcal{T})] \right\|_2 \\
&\leq \mathbb{E}_{(\bm{x}_1', \bm{x}_1) \sim \gamma_1} \left\|\Tilde{\Phi}_{\theta, 1 \to 0}(\bm{x}_1' \mid \mathcal{C}_\mathcal{T}') - \Tilde{\Phi}_{\theta, 1 \to 0}(\bm{x}_1 \mid \mathcal{C}_\mathcal{T}) \right\|_2 \text{ (by the triangle inequality)} \\
& \leq L_\Phi \left[ \mathbb{E}_{(\bm{x}_1', \bm{x}_1) \sim \gamma_1} \|\bm{x}_1' - \bm{x}_1\|_2
+ \|\mathcal{C}_\mathcal{T}' - \mathcal{C}_\mathcal{T}\|_2 \right] \\
& \leq L_{\Tilde{\bm{\mu}}_0} \|\mathcal{C}_\mathcal{T}' - \mathcal{C}_\mathcal{T}\|_2.
\end{align*}
where $L_{\Tilde{\bm{\mu}}_0} := L_\Phi(L_{p_1}+2)$.

\textbf{Step 2: Lipschitz continuity of the optimal variance}. Recall that
\begin{align*}
{\Tilde{\sigma}_0}^2(\mathcal C_{\mathcal T})
= \frac{1}{d}
\mathbb E_{\bm x_1\sim p_1(\cdot\mid \mathcal C_{\mathcal T})}
\left\|
\Tilde{\Phi}_{\theta,1\to0}(\bm x_1\mid \mathcal C_{\mathcal T})
- \Tilde{\bm{\mu}}_0(\mathcal C_{\mathcal T})
\right\|_2^2 .
\end{align*}
Equivalently,
\begin{align*}
{\Tilde{\sigma}_0}^2(\mathcal C_{\mathcal T})
= \frac{1}{d}
\left[
\mathbb E_{\bm x_1\sim p_1(\cdot\mid \mathcal C_{\mathcal T})}
\left\|
\Tilde{\Phi}_{\theta,1\to0}(\bm x_1\mid \mathcal C_{\mathcal T})
\right\|_2^2
-
\left\|
\Tilde{\bm{\mu}}_0(\mathcal C_{\mathcal T})
\right\|_2^2
\right].
\end{align*}
For any $\mathcal{C}_\mathcal{T}', \mathcal{C}_\mathcal{T} \in \mathcal{X}_{\mathcal{C}_\mathcal{T}}$, choose a coupling $\gamma_2 \in \Pi(p_1(\cdot \mid \mathcal{C}_\mathcal{T}'), p_1(\cdot \mid \mathcal{C}_{\mathcal{T}}))$ such that
\begin{align*}
    \mathbb{E}_{(\bm{x}_1', \bm{x}_1) \sim \gamma_2} \|\bm{x}_1' - \bm{x}_1\|_2^2
    \leq W_2^2\left(p_1(\cdot\mid \mathcal{C}_\mathcal{T}'), p_1(\cdot\mid \mathcal C_{\mathcal T}) \right) + \| \mathcal{C}_{\mathcal T}' - \mathcal{C}_{\mathcal T} \|_2^2
    \leq (L_{p_1}^2+1) \| \mathcal{C}_{\mathcal T}' - \mathcal{C}_{\mathcal T} \|_2^2.
\end{align*}

Then we obtain
\begin{align*}
&\left|
\mathbb E_{\bm x_1'\sim p_1(\cdot\mid \mathcal C_{\mathcal T}')}
\left\|
\Tilde{\Phi}_{\theta,1\to0}
(\bm x_1'\mid \mathcal C_{\mathcal T}')
\right\|_2^2
-
\mathbb E_{\bm x_1\sim p_1(\cdot\mid \mathcal C_{\mathcal T})}
\left\|
\Tilde{\Phi}_{\theta,1\to0}
(\bm x_1\mid \mathcal C_{\mathcal T})
\right\|_2^2
\right| \\
=&
\left|
\mathbb E_{(\bm x_1', \bm x_1)\sim \gamma_2}
\left[
\left\|
\Tilde{\Phi}_{\theta,1 \to 0}
(\bm x_1'\mid \mathcal C_{\mathcal T}')
\right\|_2^2
-
\left\|
\Tilde{\Phi}_{\theta,1 \to 0}
(\bm x_1\mid \mathcal C_{\mathcal T})
\right\|_2^2
\right]
\right| \\
\le&
\mathbb E_{(\bm x_1', \bm x_1)\sim \gamma_2}
\Big[
\left(
\left\|
\Tilde{\Phi}_{\theta,1\to0}
(\bm x_1'\mid \mathcal C_{\mathcal T}')
\right\|_2
+
\left\|
\Tilde{\Phi}_{\theta,1\to0}
(\bm x_1\mid \mathcal C_{\mathcal T})
\right\|_2
\right) \cdot
\left\|
\Tilde{\Phi}_{\theta,1\to0}
(\bm x_1'\mid \mathcal C_{\mathcal T}')
-
\Tilde{\Phi}_{\theta,1\to0}
(\bm x_1\mid \mathcal C_{\mathcal T})
\right\|_2
\Big] \\
\leq &
2 \sqrt{M} \cdot
\left(
\mathbb E_{(\bm x_1',\bm x_1)\sim \gamma_2}
\left\|
\Tilde{\Phi}_{\theta, 1 \to 0}
(\bm x_1' \mid \mathcal C_{\mathcal T}')
-
\Tilde{\Phi}_{\theta, 1 \to 0}
(\bm x_1\mid \mathcal C_{\mathcal T})
\right\|_2^2
\right)^{1/2} \text{ (by the Cauchy--Schwarz inequality)} \\
\leq & 2 L_\Phi \sqrt{2M} \left(\mathbb E_{(\bm x_1',\bm x_1)\sim \gamma_2}
\|\bm x_1'-\bm x_1\|_2^2 + \|\mathcal C_{\mathcal T}' - \mathcal C_{\mathcal T}\|_2^2\right)^\frac{1}{2} \text{ (since $\Tilde{\Phi}_{\theta, 1 \to 0}$ is Lipschitz continuous)}. \\
\leq & 2 L_\Phi \sqrt{2M(L_{p_1}^2+2)} \|\mathcal C_{\mathcal T}' - \mathcal C_{\mathcal T}\|_2.
\end{align*}

Next, for any $\mathcal{C}_\mathcal{T} \in \mathcal{X}_{\mathcal{C}_\mathcal{T}}$,
\begin{align*}
    \|\Tilde{\bm{\mu}}_0(\mathcal{C}_\mathcal{T})\|_2
    &= \|\mathbb E_{\bm{x}_1 \sim p_1(\cdot \mid \mathcal{C}_\mathcal{T})}[\Tilde{\Phi}_{\theta, 1 \to 0}(\bm{x}_1 \mid \mathcal{C}_\mathcal{T})]\|_2 \\
    &\leq \left(
    \mathbb E_{\bm{x}_1 \sim p_1(\cdot \mid \mathcal{C}_\mathcal{T})}\|\Tilde{\Phi}_{\theta, 1 \to 0}(\bm{x}_1 \mid \mathcal{C}_\mathcal{T})\|_2^2
    \right)^\frac{1}{2} \text{ (by Jensen's inequality)} \\
    &\leq \sqrt{M}.
\end{align*}

Hence
\begin{align*}
    \left| \|\Tilde{\bm{\mu}}_0(\mathcal{C}_\mathcal{T}')\|_2^2 - \|\Tilde{\bm{\mu}}_0(\mathcal{C}_\mathcal{T})\|_2^2\right|
    &\leq
    (\|\Tilde{\bm{\mu}}_0(\mathcal{C}_\mathcal{T}')\|_2 +
    \|\Tilde{\bm{\mu}}_0(\mathcal{C}_\mathcal{T})\|_2)
    \|\Tilde{\bm{\mu}}_0(\mathcal{C}_\mathcal{T}') - \Tilde{\bm{\mu}}_0(\mathcal{C}_\mathcal{T})\|_2 \\
    &\leq 2\sqrt{M} \|\Tilde{\bm{\mu}}_0(\mathcal{C}_\mathcal{T}') - \Tilde{\bm{\mu}}_0(\mathcal{C}_\mathcal{T})\|_2 
    \leq 2L_\Phi(L_{p_1}+2) \sqrt{M} \|\mathcal{C}_\mathcal{T}'-\mathcal{C}_\mathcal{T}\|_2.
\end{align*}

Therefore, we have
\begin{align*}
\left| {\Tilde{\sigma}_0}^2(\mathcal{C}_{\mathcal{T}}') - {\Tilde{\sigma}_0}^2(\mathcal{C}_{\mathcal{T}}) \right|
& \leq \frac{1}{d}
\left|
\mathbb{E}_{\bm{x}_1'\sim p_1(\cdot \mid \mathcal{C}_\mathcal{T}')}
\left\|
\Tilde{\Phi}_{\theta, 1 \to 0}(\bm{x}_1' \mid \mathcal{C}_\mathcal{T}')
\right\|_2^2
-
\mathbb{E}_{\bm{x}_1\sim p_1(\cdot \mid \mathcal{C}_\mathcal{T})}
\left\|
\Tilde{\Phi}_{\theta, 1 \to 0}(\bm{x}_1 \mid \mathcal{C}_\mathcal{T})
\right\|_2^2
\right| \\
&+ \frac{1}{d}
\left| \|\Tilde{\bm{\mu}}_0(\mathcal{C}_\mathcal{T}')\|_2^2 - \|\Tilde{\bm{\mu}}_0(\mathcal{C}_\mathcal{T})\|_2^2\right|
\leq L_{\Tilde{\sigma}_0^2} \|\mathcal{C}_\mathcal{T}'-\mathcal{C}_\mathcal{T}\|_2.
\end{align*}
where $L_{\Tilde{\sigma}_0^2} := \frac{2L_\Phi}{d} \left[\sqrt{2M(L_{p_1}^2+2)} + (L_{p_1}+2) \sqrt{M}\right]$. The proof is complete.
\end{proof}

\begin{proposition}[Lipschitz continuity of the CNP objective population minimizer]
Consider the CNP objective $\mathcal L_{\mathrm{CNP}}(\phi) = \alpha \mathcal L_{\mathrm{NLL}}(\phi) + \beta \mathcal L_{\mathrm{KL}}(\phi), \alpha > 0, \beta \geq 0$. Let $s(\mathcal{C}_\mathcal{T}) = \sigma_\phi^{*2}(\mathcal{C}_\mathcal{T})$ denote the unique positive solution of $F(s,\mathcal{C}_\mathcal{T})=0$, where
\begin{equation*}
F(s, \mathcal{C}_\mathcal{T})
=\beta s^2 + (\alpha-\beta)s - \alpha \Tilde{\sigma}_0^2(\mathcal{C}_\mathcal{T})-
\frac{\alpha\beta^2s^2}{d(\alpha+\beta s)^2}
\|\Tilde{\bm{\mu}}_0(\mathcal{C}_\mathcal{T})\|^2.
\end{equation*}
Assume that for all $\mathcal{C}_\mathcal{T} \in \mathcal{X}_{\mathcal{C}_\mathcal{T}}$,
\begin{enumerate}[label=(\arabic*)]
    \item $\|\Tilde{\bm{\mu}}_0(\mathcal{C}_\mathcal{T})\|_2 \leq M_{\Tilde{\bm{\mu}}_0}$,
    \item $\Tilde{\bm{\mu}}_0(\cdot)$ is $L_{\Tilde{\bm{\mu}}_0}$-Lipschitz continuous on $\mathcal{X}_{\mathcal{C}_\mathcal{T}}$,
    \item $\Tilde{\sigma}_0^2(\cdot)$ is $L_{\Tilde{\sigma}_0^2}$-Lipschitz continuous on $\mathcal{X}_{\mathcal{C}_\mathcal{T}}$,
    \item $0 < s_{\min}\le s(\mathcal{C}_\mathcal{T}) \leq s_{\max} < \infty$, and
    \item there exists $\kappa > 0$ such that
    $\left|\partial_s F(s, \mathcal{C}_\mathcal{T})\right| \geq \kappa$ for all $s \in [s_{\min}, s_{\max}]$.
\end{enumerate}
Then $\bm{\mu}_\phi^*(\mathcal{C}_\mathcal{T})$ and ${\sigma_\phi^*}^2(\mathcal{C}_{\mathcal{T}})$ are Lipschitz continuous on $\mathcal{X}_{\mathcal{C}_\mathcal{T}}$.

\end{proposition}

\begin{proof}
For any two conditions $\mathcal{C}_\mathcal{T}, \mathcal{C}_\mathcal{T}' \in \mathcal{X}_{\mathcal{C}_\mathcal{T}}$, we have
\begin{equation*}
    F(s(\mathcal{C}_\mathcal{T}), \mathcal{C}_\mathcal{T}) = F(s(\mathcal{C}_\mathcal{T}'), \mathcal{C}_\mathcal{T}') = 0.
\end{equation*}
Therefore,
\begin{equation*}
    F(s(\mathcal{C}_\mathcal{T}), \mathcal{C}_\mathcal{T}) - F(s(\mathcal{C}_\mathcal{T}'), \mathcal{C}_\mathcal{T})
    = F(s(\mathcal{C}_\mathcal{T}'), \mathcal{C}_\mathcal{T}') - F(s(\mathcal{C}_\mathcal{T}'), \mathcal{C}_\mathcal{T}).
\end{equation*}
By the Mean Value Theorem, there exists $\bar{s}$ between $s(\mathcal{C}_\mathcal{T})$ and $s(\mathcal{C}_\mathcal{T}')$ such that
\begin{equation*}
    F(s(\mathcal{C}_\mathcal{T}), \mathcal{C}_\mathcal{T}) - F(s(\mathcal{C}_\mathcal{T}'), \mathcal{C}_\mathcal{T}) = \partial_{s} F(\bar{s}, \mathcal{C}_\mathcal{T}) (s(\mathcal{C}_\mathcal{T}) - s(\mathcal{C}_\mathcal{T}')).
\end{equation*}
Since $s(\mathcal{C}_\mathcal{T})$ and 
$s(\mathcal{C}_\mathcal{T}')$ both lie in $[s_{\min},s_{\max}]$, we have $\bar{s}\in[s_{\min},s_{\max}]$. Thus,
\begin{equation*}
    \kappa |s(\mathcal{C}_\mathcal{T}) - s(\mathcal{C}_\mathcal{T}')| \leq |F(s(\mathcal{C}_\mathcal{T}'), \mathcal{C}_\mathcal{T}') - F(s(\mathcal{C}_\mathcal{T}'), \mathcal{C}_\mathcal{T})|.
\end{equation*}
For any $ s \in[s_{\min}, s_{\max}]$, we have
\begin{align*}
|F(s, \mathcal{C}_\mathcal{T}) - F(s, \mathcal{C}_\mathcal{T}')|
&\leq \alpha |\Tilde{\sigma}_0^2(\mathcal{C}_\mathcal{T}) - \Tilde{\sigma}_0^2(\mathcal{C}_\mathcal{T}')|
+ \frac{\alpha\beta^2s^2}{d(\alpha+\beta s)^2}
\left| \|\Tilde{\bm{\mu}}_0(\mathcal{C}_\mathcal{T})\|_2^2 - \|\Tilde{\bm{\mu}}_0(\mathcal{C}_\mathcal{T}')\|_2^2 \right| \\
&\leq \alpha L_{\Tilde{\sigma}_0^2} \|\mathcal{C}_\mathcal{T} - \mathcal{C}_\mathcal{T}'\|_2 + \frac{\beta^2 s_{\max}^2}{d \alpha} (2M_{\Tilde{\bm{\mu}}_0} L_{\Tilde{\bm{\mu}}_0}) \|\mathcal{C}_\mathcal{T} - \mathcal{C}_\mathcal{T}'\|_2 \\
&= \left(\alpha L_{\Tilde{\sigma}_0^2} + \frac{2\beta^2 s_{\max}^2M_{\Tilde{\bm{\mu}}_0} L_{\Tilde{\bm{\mu}}_0}}{d \alpha}\right) \|\mathcal{C}_\mathcal{T} - \mathcal{C}_\mathcal{T}'\|_2.
\end{align*}
Hence $|s(\mathcal{C}_\mathcal{T}) - s(\mathcal{C}_\mathcal{T}')| \leq L_s \|\mathcal{C}_\mathcal{T} - \mathcal{C}_\mathcal{T}'\|_2$,
where $L_s = \frac{1}{\kappa}\left(\alpha L_{\Tilde{\sigma}_0^2} + \frac{2\beta^2 s_{\max}^2M_{\Tilde{\bm{\mu}}_0} L_{\Tilde{\bm{\mu}}_0}}{d \alpha}\right)$. Therefore, ${\sigma_\phi^*}^2(\mathcal{C}_\mathcal{T})$ is Lipschitz continuous on $\mathcal{X}_{\mathcal{C}_\mathcal{T}}$.
Next, by Prop.~B.1, the optimal mean satisfies
\begin{equation*}
    \bm{\mu}_\phi^*(\mathcal{C}_\mathcal{T}) = \frac{\alpha}{\alpha+\beta s(\mathcal{C}_\mathcal{T})} \Tilde{\bm{\mu}}_0(\mathcal{C}_\mathcal{T}).
\end{equation*}
Therefore,
\begin{align*}
\|\bm{\mu}_\phi^*(\mathcal{C}_\mathcal{T}) - \bm{\mu}_\phi^*(\mathcal{C}_\mathcal{T}')\|_2
&=
\left\|
\frac{\alpha}{\alpha + \beta s(\mathcal{C}_\mathcal{T})}\Tilde{\bm{\mu}}_0(\mathcal{C}_\mathcal{T}) -
\frac{\alpha}{\alpha + \beta s(\mathcal{C}_\mathcal{T}')}\Tilde{\bm{\mu}}_0(\mathcal{C}_\mathcal{T}')
\right\|_2 \\
& \leq
\frac{\alpha}{\alpha+\beta s(\mathcal{C}_\mathcal{T})}
\|\Tilde{\bm{\mu}}_0(\mathcal{C}_\mathcal{T}) - \Tilde{\bm{\mu}}_0(\mathcal{C}_\mathcal{T}')\|_2 + \alpha
\left|\frac{1}{\alpha+\beta s(\mathcal{C}_\mathcal{T})} - \frac{1}{\alpha + \beta s(\mathcal{C}_\mathcal{T}')}
\right| \|\Tilde{\bm{\mu}}_0(\mathcal{C}_\mathcal{T}')\|_2.
\end{align*}
Moreover,
\begin{align*}
\left|\frac{1}{\alpha + \beta s(\mathcal{C}_\mathcal{T})} - \frac{1}{\alpha + \beta s(\mathcal{C}_\mathcal{T}')} \right|
&= \frac{\beta |s(\mathcal{C}_\mathcal{T}) - s(\mathcal{C}_\mathcal{T}')|}
{(\alpha + \beta s(\mathcal{C}_\mathcal{T}))(\alpha + \beta s(\mathcal{C}_\mathcal{T}'))} \\
&\leq \frac{\beta}{\alpha^2}|s(\mathcal{C}_\mathcal{T}) - s(\mathcal{C}_\mathcal{T}')|
\leq \frac{\beta L_s}{\alpha^2} \|\mathcal{C}_\mathcal{T} - \mathcal{C}_\mathcal{T}'\|_2.
\end{align*}
Hence $\|\bm{\mu}_\phi^*(\mathcal{C}_\mathcal{T}) - \bm{\mu}_\phi^*(\mathcal{C}_\mathcal{T}')\|_2 \leq L_{\bm{\mu}_\phi^*} \|\mathcal{C}_\mathcal{T} - \mathcal{C}_\mathcal{T}'\|_2$,
where $L_{\bm{\mu}_\phi^*} = L_{\Tilde{\bm{\mu}}_0} + \frac{\beta L_s M_{\Tilde{\bm{\mu}}_0}}{\alpha}$. Therefore, $\bm{\mu}_\phi^*(\mathcal{C}_\mathcal{T})$ is Lipschitz continuous on $\mathcal{X}_{\mathcal{C}_\mathcal{T}}$. The proof is complete.
\end{proof}

\begin{lemma}[High-probability two-sided bound for Gaussian sample distance]
\label{prop:gaussian_sample}
Let $\bm{x} \sim \mathcal{N}(\bm{\mu}, {\sigma}^2\mathbf{I}_d), \bm{x}' \sim \mathcal{N}(\bm{\mu}', {\sigma'}^2\mathbf{I}_d)$ be independent random vectors in $\mathbb{R}^d$ with $\sigma >0$ and $\sigma' > 0$. Then, for any $\delta \in (0, 1)$, with probability at least $1-\delta$,
\begin{align*}
\|\bm{\mu} - \bm{\mu}'\|_2 -
\sqrt{\sigma^2 + {\sigma'}^2} c_{d, \delta}
\leq \| \bm{x} - \bm{x}' \|_2 \leq
\| \bm{\mu} - \bm{\mu}' \|_2 +
\sqrt{\sigma^2 + {\sigma'}^2} c_{d, \delta},
\end{align*}
where $c_{d, \delta}=\sqrt{d + 2\sqrt{d \log \frac{1}{\delta}} + 2\log \frac{1}{\delta}}$.
\end{lemma}

\begin{proof}
Since $\bm{x}$ and $\bm{x}'$ are independent Gaussian random vectors,
\begin{align*}
\bm{x} - \bm{x}' \sim \mathcal{N}(\bm{\mu} - \bm{\mu}', (\sigma^2 + {\sigma'}^2) \mathbf{I}_d).
\end{align*}
Therefore, we obtain
\begin{align*}
\left\|
\frac{(\bm{x} - \bm{x}') - (\bm{\mu} - \bm{\mu}')}{\sqrt{\sigma^2 + {\sigma'}^2}} \right\|_2^2
\sim \chi^2(d).
\end{align*}

By Laurent--Massart bounds of chi-square distribution with degrees of freedom $d$,
\begin{align*}
\mathbb{P}\left( \left\| \frac{(\bm{x} - \bm{x}') - (\bm{\mu} - \bm{\mu}')}{\sqrt{\sigma^2 + {\sigma'}^2}} \right\|_2^2
\leq d + 2\sqrt{d \tau} + 2\tau \right)
\geq 1 - \mathrm{e}^{-\tau} \text{ for all $\tau>0$}.
\end{align*}

Setting $\tau = \log \frac{1}{\delta} >0$ and $c_{d, \delta}=\sqrt{d + 2\sqrt{d \log \frac{1}{\delta}} + 2\log \frac{1}{\delta}}$, then we have
\begin{align*}
\mathbb{P}\left( \|(\bm{x} - \bm{x}') - (\bm{\mu} - \bm{\mu}') \|_2
\leq c_{d, \delta}\sqrt{\sigma^2 + {\sigma'}^2} \right)
\geq 1 - \delta.
\end{align*}

Hence, with probability at least $1-\delta$, by the reverse triangle inequality,
\begin{align*}
\left|\|\bm{x} - \bm{x}'\|_2 - \|\bm{\mu} - \bm{\mu}'\|_2\right|
\leq \left\|(\bm{x} - \bm{x}') - (\bm{\mu} - \bm{\mu}')\right\|_2 \leq \sqrt{\sigma^2 + {\sigma'}^2} c_{d, \delta}.
\end{align*}

Therefore, we have
\begin{align*}
\|\bm{\mu} - \bm{\mu}'\|_2 -
\sqrt{\sigma^2 + {\sigma'}^2} c_{d, \delta}
\leq \| \bm{x} - \bm{x}' \|_2 \leq
\| \bm{\mu} - \bm{\mu}' \|_2 +
\sqrt{\sigma^2 + {\sigma'}^2} c_{d, \delta}.
\end{align*}

\end{proof}

Lemma~\ref{prop:gaussian_sample} establishes that the distance between two independent Gaussian samples can be bounded in terms of the distance between their means and the scale of their standard deviations. Rmk.~\ref{rmk:gaussian_sample} separately examines the implications of this bound for standard Gaussian samples and conditional noise prediction (CNP).

\begin{remark}
\label{rmk:gaussian_sample}
Given the condition $\mathcal{C}_\mathcal{T}, \mathcal{C}_\mathcal{T}' \in \mathcal{X}_{\mathcal{C}_\mathcal{T}}$ and $\delta \in (0, 1)$, with probability at least $1-\delta$,
\begin{enumerate}[leftmargin=*]
    \item Standard Gaussian: If $\bm{x}_0^\mathrm{SG}, {\bm{x}_0^\mathrm{SG}}' \overset{\mathrm{iid}}{\sim} \mathcal{N}(\bm{0}, \mathbf{I}_d)$, then
    \begin{align*}
        \| \bm{x}_0^\mathrm{SG} - {\bm{x}_0^\mathrm{SG}}' \|_2 = \mathcal{O}(\sqrt{2} c_{d, \delta}).
    \end{align*}
    \item CNP: If $\bm{x}_0^\mathrm{CNP} \sim \mathcal{N}(\bm{\mu}_\phi^*(\mathcal{C}_\mathcal{T}), {\sigma_\phi^*}^2(\mathcal{C}_\mathcal{T})\mathbf{I}_d)$ and ${\bm{x}_0^\mathrm{CNP}}' \sim \mathcal{N}(\bm{\mu}_\phi^*(\mathcal{C}_\mathcal{T}'), {\sigma_\phi^*}^2(\mathcal{C}_\mathcal{T}') \mathbf{I}_d)$ are drawn independently, then
    \begin{align*}
        \| \bm{x}_0^\mathrm{CNP} - {\bm{x}_0^\mathrm{CNP}}' \|_2 = \|\bm{\mu}_\phi^*(\mathcal{C}_\mathcal{T})-\bm{\mu}_\phi^*(\mathcal{C}_\mathcal{T}')\|_2 \pm \mathcal{O}\left( \sqrt{{\sigma_\phi^*}^2(\mathcal{C}_\mathcal{T}) + {\sigma_\phi^*}^2(\mathcal{C}_\mathcal{T}')} c_{d, \delta} \right).
    \end{align*}
\end{enumerate}
\end{remark}

\begin{theorem}
\label{cor:gaussian_sample}
Suppose $\bm{\mu}_\phi^*(\cdot)$ is $L_{\bm{\mu}_\phi^*}$-Lipschitz continuous on $\mathcal{X}_{\mathcal{C}_\mathcal{T}}$. Given $\mathcal{C}_\mathcal{T}, \mathcal{C}_\mathcal{T}' \in \mathcal{X}_{\mathcal{C}_\mathcal{T}}$,
if $\bm{x}_0^\mathrm{CNP} \sim \mathcal{N}(\bm{\mu}_\phi^*(\mathcal{C}_\mathcal{T}), {\sigma_\phi^*}^2(\mathcal{C}_\mathcal{T})\mathbf{I}_d)$ and  ${\bm{x}_0^\mathrm{CNP}}' \sim \mathcal{N}(\bm{\mu}_\phi^*(\mathcal{C}_\mathcal{T}')), {\sigma_\phi^*}^2(\mathcal{C}_\mathcal{T}') \mathbf{I}_d)$, then with probability at least $1-\delta$,
\begin{align*}
    \| \bm{x}_0^\mathrm{CNP} - {\bm{x}_0^\mathrm{CNP}}' \|_2 \leq L_{\bm{\mu}_\phi^*} \|\mathcal{C}_\mathcal{T}-\mathcal{C}_\mathcal{T}'\|_2 + \sqrt{{\sigma_\phi^*}^2(\mathcal{C}_\mathcal{T}) + {\sigma_\phi^*}^2(\mathcal{C}_\mathcal{T}')} c_{d, \delta}.
\end{align*}
\end{theorem}

Rmk.~\ref{rmk:gaussian_sample} reveals that the discrepancy between two CNP samples can be decomposed into two components: a deterministic shift induced by the difference between their conditions, and a stochastic deviation induced by the predicted uncertainty. In particular, Theorem~\ref{cor:gaussian_sample} shows that when the conditional mean predictor is Lipschitz continuous, nearby conditions lead to nearby predicted means. Therefore, the sample-level discrepancy remains controlled with high probability as long as the conditions are close and the predicted variances are small.

\begin{proposition}
\label{prop:cnp_target_var}
Given the condition $\mathcal{C}_\mathcal{T}$, let $\bm{x}_0^\mathrm{SG} \sim \mathcal{N}(\mathbf{0}, \mathbf{I})$ and $\bm{x}_0^\mathrm{CNP} \sim p_\phi^*(\cdot\mid\mathcal{C}_\mathcal{T})$, then

\begin{equation*}
    \operatorname{trace}(\operatorname{Cov}(\bm{x}_1 - \bm{x}_0^\mathrm{CNP} \mid \mathcal{C}_\mathcal{T})) - 
    \operatorname{trace}(\operatorname{Cov}(\bm{x}_1 - \bm{x}_0^\mathrm{SG} \mid \mathcal{C}_\mathcal{T}))
    = d(\sigma_\phi^2(\mathcal{C}_\mathcal{T}) - 1).
\end{equation*}
\end{proposition}
\begin{proof}
Since $\bm{x}_0\mid\mathcal{C}_\mathcal{T}$ and $\bm{x}_1\mid\mathcal{C}_\mathcal{T}$ are independent,
we have
\begin{equation*}
    \operatorname{Cov}(\bm{x}_1 - \bm{x}_0 \mid \mathcal{C}_\mathcal{T}) = \operatorname{Cov}(\bm{x}_1 \mid \mathcal{C}_\mathcal{T}) + \operatorname{Cov}(\bm{x}_0 \mid \mathcal{C}_\mathcal{T}).
\end{equation*}
Moreover,
\begin{equation*}
    \operatorname{Cov}(\bm{x}_0^\mathrm{CNP} \mid \mathcal{C}_\mathcal{T}) = {\sigma_\phi^*}^2(\mathcal{C}_\mathcal{T}) \mathbf{I}_d,
    \quad
    \operatorname{Cov}(\bm{x}_0^\mathrm{SG} \mid \mathcal{C}_\mathcal{T}) = \mathbf{I}_d
\end{equation*}
Therefore, the proof is conducted.
\end{proof}

Prop.~\ref{prop:cnp_target_var} shows that replacing the standard Gaussian source with the CNP source directly changes the variance of the FM target velocity. In particular, when the learned CNP variance satisfies $\sigma_\phi^2(\mathcal{C}_\mathcal{T}) < 1$, which typically holds when the temporal condition strongly constrains the target sample, the covariance trace of the target velocity is reduced compared to standard Gaussian initialization. This indicates that CNP provides a more concentrated source distribution, yielding lower-variance velocity targets.

\section{Experimental Details}
\label{app:experimental_config}

\subsection{Video Generation}
\label{app:video_generation_config}


\paratitle{Compute environment.}
Video prediction and interpolation experiments are conducted on the AMD AI \& HPC Clusters Taiwan, equipped with AMD Instinct MI300X GPUs with 192GB HBM3 memory and AMD EPYC 9684X 96-Core processors. GPUs within each node are interconnected through AMD Infinity Fabric (XGMI). Flow reversal and final evaluation are run on a single 8-GPU MI300X node, while CNP training and FM fine-tuning use a single MI300X GPU. All AMD experiments use ROCm 7.2.0. Flow matching pretraining and fine-tuning use BF16 mixed precision.

\paratitle{Training time.}
The wall-clock runtime varies with dataset size and video resolution. For KTH, BAIR64, and CLEVRER, the additional offline cost of the CNP-Flow pipeline is approximately $5.5$, $9.1$, and $17$ hours, respectively, relative to vanilla fixed-source FM. On CLEVRER, our largest-resolution video benchmark, source retrieval via flow reversal takes approximately $3$ hours on a single 8-GPU MI300X node, while CNP training and FM fine-tuning take approximately $14$ hours on a single MI300X GPU. Including final evaluation, each video dataset/task completes within 24 hours on the specified hardware.

\paratitle{Computational overhead.}
To quantify the additional cost of \method{} at generation time, we measure inference latency, peak memory, and model FLOPs under same hardware, batch-size, and ODE-solver settings. As summarized in Tab.~\ref{tab:app_compute_overhead}, \method{} increases inference latency by only about $3.4\%$. The CNP is evaluated once per generated frame and requires $0.297$ GFLOPs; at NFE${=}10$, this corresponds to approximately $0.34\%$ additional model FLOPs, with about $15$ MB additional peak memory. The $98.08$M-parameter FM backbone is augmented with a $47.69$M-parameter CNP ($50.84$M on CLEVRER). Source samples are retrieved offline via flow reversal and cached, and flow reversal is never invoked during generation.

\begin{table}[t]
\centering
\caption{\textbf{Inference overhead of \method{}.}
Latency is measured over 1,000 generated frames under matched hardware,
batch-size, and solver settings.}
\label{tab:app_compute_overhead}
\vspace{3pt}

\footnotesize
\setlength{\tabcolsep}{5pt}
\renewcommand{\arraystretch}{0.95}

\resizebox{\linewidth}{!}{
\begin{tabular}{lcccc}
\toprule
Dataset
& Fixed-source FM
& \method{}
& Latency increase
& Additional memory \\
\midrule
KTH / BAIR64
& 207.1 ms/frame
& 214.2 ms/frame
& $3.43\%$
& \multirow{2}{*}{${\sim}15$ MB} \\

CLEVRER
& 211.9 ms/frame
& 219.1 ms/frame
& $3.39\%$
& \\
\bottomrule
\end{tabular}
}

\end{table}

\paratitle{Backbone and latent space.}
All video experiments are performed in the VQ-GAN latent space released by RIVER~\citep{davtyan2023river}. The latent resolution is $8{\times}8$ for BAIR64 and KTH, and $16{\times}16$ for CLEVRER at $128{\times}128$. For video prediction, we reuse the released RIVER prediction checkpoints for KTH and CLEVRER, while retraining the BAIR64 fixed-source FM backbone due to checkpoint incompatibility with our evaluation setup. For video interpolation, since no compatible pretrained interpolation checkpoints are provided, we train the fixed-source FM backbone from scratch. Detailed FM fine-tuning and flow reversal configurations are given in Tab.~\ref{tab:app_video_pred_fm_config} and~\ref{tab:app_video_interp_fm_config}.

\begin{table*}[t]
\centering
\caption{Video prediction fixed-source FM and flow reversal configuration. Source samples obtained through flow reversal are cached only for CNP supervision and are not used at inference time.}
\label{tab:app_video_pred_fm_config}
\begin{tabular}{l c c c}
\toprule
Component & KTH & CLEVRER & BAIR64 \\
\midrule
Flow reversal steps & 100 & 20 & 50 \\
FM fine-tune length & 200k steps & 20k steps & 50k steps \\
FM fine-tune batch & 64 & 128 & 64 \\
FM fine-tune LR & $5{\times}10^{-5}$ & $5{\times}10^{-5}$ & $5{\times}10^{-5}$ \\
\bottomrule
\end{tabular}

\end{table*}
\begin{table*}[t]
\centering
\caption{Video interpolation fixed-source FM and flow reversal configuration. Since no compatible interpolation checkpoint is available, the fixed-source CFM backbone is trained from scratch.}
\label{tab:app_video_interp_fm_config}
\begin{tabular}{l c c}
\toprule
Component & KTH & CLEVRER \\
\midrule
FM pretrain steps & 300k & 300k \\
FM pretrain batch & 32 & 16 \\
Flow reversal steps & 100 & 40 \\
FM fine-tune steps & 200k & 200k \\
FM fine-tune batch & 32 & 16 \\
FM fine-tune LR & $1{\times}10^{-4}$ & $1{\times}10^{-4}$ \\
\bottomrule
\end{tabular}

\end{table*}

Source samples obtained through flow reversal are cached only for CNP supervision and are not used at inference time. FM fine-tuning uses Adam with $\beta_1=0.9$, $\beta_2=0.999$, no weight decay, and gradient clipping with maximum norm 1.0.

\paratitle{CNP architecture and training.}
The video CNP is a convolutional encoder over latent feature grids. It predicts the mean and log-variance of an isotropic Gaussian source distribution. For prediction, the CNP conditions on the available context latents and temporal-distance embedding. For interpolation, the CNP conditions on the local reference latent, the endpoint latent, and the temporal-distance embedding. The shared architecture is summarized in Tab.~\ref{tab:app_video_cnp_arch}. All video CNPs are trained with Adam using a learning rate of $1 {\times}10^{-4}$ and a KL weight of $\lambda_{\mathrm{KL}}=1{\times}10^{-4}$. Task-specific training lengths, batch sizes, and KL warm-up schedules are listed in Tab.~\ref{tab:app_video_cnp_training_config}.

\begin{table*}[t]
\centering
\caption{Video CNP architecture. The same architecture is used for video prediction and video interpolation.}
\label{tab:app_video_cnp_arch}
\begin{tabular}{l c}
\toprule
Component & Setting \\
\midrule
Backbone & Convolutional encoder \\
Input & Conditional latents + temporal-distance embedding \\
Output & $\mu,\log\sigma^2$ \\
Base channels & 64 \\
Channel multipliers & $(1,2,4,8)$ \\
Residual blocks per resolution & 2 \\
Attention resolution & Middle latent resolution \\
Dropout & 0.1 \\
Latent grid & $8{\times}8$ or $16{\times}16$ \\
\bottomrule
\end{tabular}

\end{table*}
\begin{table*}[t]
\centering
\caption{Video CNP training configuration. All video CNPs are trained with Adam using learning rate $1{\times}10^{-4}$ and KL weight $\lambda_{\mathrm{KL}}=1{\times}10^{-4}$.}
\label{tab:app_video_cnp_training_config}
\begin{tabular}{l c c c c}
\toprule
Task & Dataset & Length & Batch & KL warm-up \\
\midrule
Prediction & KTH & 50 epochs & 64 & 10 epochs \\
Prediction & CLEVRER & 15 epochs & 1024 & 10 epochs \\
Prediction & BAIR64 & 15 epochs & 64 & 5 epochs \\
\midrule
Interpolation & KTH & 200k steps & 32 & 5 epochs \\
Interpolation & CLEVRER & 200k steps & 16 & 2 epochs \\
\bottomrule
\end{tabular}

\end{table*}

\paratitle{Source-learning baselines.} For CAR-Flow~\citep{chen2025car}, we use its source-only variant, which modifies the source while leaving the target unchanged. For CSFM~\citep{kim2026better}, we adapt its source predictor to the same temporal condition used by the RIVER backbone and retain its original regularization and source--target alignment objectives. All source-learning baselines use the same FM backbone, training data, pretrained FM initialization, fine-tuning budget, sampling NFE, and evaluation protocol as \method{}. Method-specific hyperparameters follow the original recommendations or are selected on validation data without using test-set performance.

\paratitle{Evaluation.}
For both video prediction and video interpolation, we evaluate 256 videos and generate 100 samples per video. PSNR and SSIM are reported as best-of-100 scores, while FVD is computed over the full generated sample pool using the same I3D feature extractor as RIVER~\citep{davtyan2023river}. Unless otherwise stated, all main-text results in Tab.~\ref{tab:video_prediction} and~\ref{tab:video_interpolation} use a 10-step Euler solver (NFE${=}10$); Fig.~\ref{fig:efficiency} sweeps NFE separately to characterize inference efficiency.

\paratitle{Variability across runs.}
To evaluate run-to-run variability, we run the complete pipeline with three independent random seeds for each setting and report mean $\pm$ standard deviation. Tab.~\ref{tab:app_run_variability} provides the complete three-seed results for both video prediction and interpolation across all evaluated datasets.

\begin{table}[t]
\centering
\caption{\textbf{Variability across runs for \method{}.}
Results are reported as mean $\pm$ standard deviation over three
independent runs for video prediction and interpolation.}
\label{tab:app_run_variability}

\footnotesize
\setlength{\tabcolsep}{6pt}
\renewcommand{\arraystretch}{1.05}

\begin{tabular}{lccc}
\toprule
Dataset
& PSNR $\uparrow$
& SSIM $\uparrow$
& FVD $\downarrow$ \\
\midrule

\multicolumn{4}{l}{\textit{Video prediction}} \\

BAIR64 $1\!\rightarrow\!15$
& $20.15 \pm 0.13$
& $0.8390 \pm 0.0021$
& $61.17 \pm 1.58$ \\

KTH $10\!\rightarrow\!30$
& $34.25 \pm 0.26$
& $0.9432 \pm 0.0025$
& $44.76 \pm 1.53$ \\

KTH $10\!\rightarrow\!40$
& $32.69 \pm 0.22$
& $0.9303 \pm 0.0040$
& $54.88 \pm 3.85$ \\

CLEVRER $2\!\rightarrow\!14$
& $42.23 \pm 0.19$
& $0.9935 \pm 0.0002$
& $24.18 \pm 2.46$ \\

\midrule
\multicolumn{4}{l}{\textit{Video interpolation}} \\

KTH
& $36.07 \pm 0.54$
& $0.9634 \pm 0.0021$
& $36.81 \pm 2.65$ \\

CLEVRER
& $45.68 \pm 0.01$
& $0.9954 \pm 0.0001$
& $10.77 \pm 0.02$ \\

\bottomrule
\end{tabular}

\end{table}

\subsection{Motion Planning}
\label{app:motion_planning_config}

\paratitle{Compute environment.}
Motion planning experiments are conducted on a workstation equipped with NVIDIA RTX 6000 Ada generation GPUs with 48GB memory. Each training, flow reversal, and evaluation run uses a single GPU. We build on the public MPD implementation~\citep{carvalho2025motion} with PyTorch and CUDA. Training uses EMA weights and BF16 automatic mixed precision, while flow reversal and evaluation are performed in FP32.

\paratitle{Training time.}
The complete CNP-Flow pipeline completes within 12 hours on this workstation: fixed-source FM pretraining ($\sim 4$ hours), source retrieval and caching via flow reversal ($\sim 1$ hour, run once), CNP training ($\sim 2$ hours), and FM fine-tuning ($\sim 4$ hours), with the remaining time spent on final evaluation. The cost of source retrieval and caching is amortized across all subsequent CNP epochs, and the overall pipeline takes roughly $3{\times}$ the training time of a vanilla fixed-source FM.

\paratitle{Task and trajectory representation.}
We evaluate on the Panda Warehouse benchmark from MPD~\citep{carvalho2025motion}, which considers 7-DoF Franka Panda joint-space planning with end-effector pose goals. We use the same demonstration dataset, robot model, collision environment, and evaluation protocol as MPD. Trajectories are parameterized as cubic B-splines with 22 desired control points along the trajectory, together with 2 phantom control points used to enforce boundary conditions, 
giving 24 learnable control points in total. The spline is evaluated at 128 time samples over a 10-second horizon.

\paratitle{Backbone architecture.}
The diffusion baseline and the fixed-source flow matching backbone share the same temporal U-Net and conditioning architecture. The trajectory backbone predicts $24$ learnable B-spline control points. The U-Net uses a base channel size of 32 with channel multipliers $(1,2,4,8)$. The start state and goal pose are encoded by separate 2-layer MLPs, and both are projected into a 128-dimensional context embedding.

\paratitle{Motion planning CNP.}
The motion planning CNP is a residual MLP over vectorized B-spline trajectory control points. It predicts the mean and log-variance of an isotropic Gaussian source distribution over the learnable control-point noise. The architecture is summarized in Tab.~\ref{tab:app_mp_cnp_arch}. We train the CNP for 2000 epochs with batch size 512, learning rate $1{\times}10^{-4}$, and KL weight $\lambda_{\mathrm{KL}}=1{\times}10^{-6}$. The KL weight is linearly warmed up for the first 50 epochs. Optimization uses Adam with $\beta_1=0.9$, $\beta_2=0.999$, no weight decay, gradient clipping with maximum norm 1.0, and FP32 precision.

\begin{table*}[t]
\centering
\caption{\textbf{Motion planning CNP architecture.} The CNP predicts an isotropic Gaussian source distribution over the learnable B-spline control-point noise.}
\label{tab:app_mp_cnp_arch}
\begin{tabular}{l c}
\toprule
Component & Setting \\
\midrule
Backbone & Residual MLP \\
Input & Start state + EE goal pose + context embedding \\
Raw condition dimension & 26 \\
Context embedding dimension & 128 \\
Output & $\mu,\log\sigma^2$ \\
Hidden dimension & 1024 \\
Residual blocks & 6 \\
Activation & GELU \\
Output initialization & Zero-initialized head \\
EMA decay & 0.995 \\
\bottomrule
\end{tabular}

\end{table*}

\paratitle{Training configuration.}
For the diffusion baseline, we use the released MPD checkpoint. The fixed-source flow matching model is pretrained for $6.0{\times}10^5$ steps with batch size 512, learning rate $3{\times}10^{-4}$, EMA weights, and BF16 automatic mixed precision. Source samples obtained through flow reversal are then cached once for CNP supervision using 200 Euler steps with batch size 1024 in FP32; these cached targets are not used during inference. After CNP training, we freeze the CNP and fine-tune the flow matching backbone for 200k steps using CNP-sampled source noise instead of the standard Gaussian source, with batch size 512, learning rate $1.5{\times}10^{-4}$, EMA weights, and BF16 automatic mixed precision.


\paratitle{Evaluation.}
For evaluation, we follow the MPD inference protocol. 
All methods are evaluated on the same 100 test start-goal pairs, and results are averaged over three independent seeds.
For each pair, we generate $N=50$ candidate trajectories and select the shortest valid trajectory. 
The MPD baseline uses its cost-guided DDIM sampler, which interleaves denoising with cost-gradient updates: we follow the paper and use $15$ DDIM steps, activate guidance during the final $20\%$ of the denoising process. Specifically, the last 3 denoising steps each perform $n_{\text{guide}}=4$ cost-gradient updates, resulting in 12 gradient updates in total.

For \method{}, we use a 15-step Euler ODE sampler followed by 12 post-hoc cost-gradient refinement steps, using the same cost terms and weights as MPD. Thus, both methods use 15 generative steps and 12 cost-gradient updates.

Specifically, after ODE sampling, we run $12$ refinement steps with learning rate $0.01$, maximum perturbation $0.15$, and value clipping at $1.0$, using the same cost terms and weights as MPD.
This keeps the cost-refinement stage matched to the released MPD configuration, while replacing the diffusion prior with our learned-source flow matching sampler. Flow reversal is used only to construct CNP supervision and is not invoked during sampling or evaluation. We report the same evaluation metrics as MPD: \emph{Success}, \emph{Valid fraction}, end-effector position/orientation error, and \emph{collision intensity}. We use \emph{Valid fraction} as the primary metric, as \emph{Success} saturates with $N=50$ candidates per task. Following MPD, \emph{collision intensity} is computed as the fraction of evaluated trajectory waypoints that are in collision.

\paratitle{Performance without post-hoc refinement.} 
\begin{table}[t]
\centering
\caption{\textbf{Motion-planning performance without post-hoc refinement.}
Results are reported as mean $\pm$ standard deviation over three
independent runs.}
\label{tab:app_mp_no_refinement}

\footnotesize
\setlength{\tabcolsep}{4pt}
\renewcommand{\arraystretch}{1.05}

\begin{tabular}{lcccc}
\toprule
Method
& Valid (\%) $\uparrow$
& EE pos. (cm) $\downarrow$
& EE ori. ($^\circ$) $\downarrow$
& Collision $\downarrow$ \\
\midrule

Diffusion prior
& $64.3 \pm 1.5$
& $5.3 \pm 0.1$
& $7.5 \pm 0.5$
& $0.0402 \pm 0.0013$ \\

Fixed-source FM
& $61.1 \pm 3.8$
& $3.8 \pm 0.2$
& $5.6 \pm 0.5$
& $0.0445 \pm 0.0034$ \\

\cellcolor{ourshl}\textbf{\method{}}
& \cellcolor{ourshl}$64.3 \pm 4.6$
& \cellcolor{ourshl}$\mathbf{2.1 \pm 0.2}$
& \cellcolor{ourshl}$\mathbf{2.9 \pm 0.3}$
& \cellcolor{ourshl}$\mathbf{0.0366 \pm 0.0045}$ \\

\bottomrule
\end{tabular}
\end{table}
To verify that the gains of \method{} are not limited to settings with post-hoc refinement, we additionally evaluate trajectories without the refinement stage. As shown in Tab.~\ref{tab:app_mp_no_refinement}, \method{} reduces the end-effector position and orientation errors from $3.8$ cm and $5.6^\circ$ with fixed-source FM to $2.1$ cm and $2.9^\circ$, respectively. This indicates that the improvement is already present before the post-hoc optimization stage.

\section{Large-Scale Physical World-Model Benchmark}
\label{app:acwm_phys}

To evaluate whether the benefit of \method{} extends beyond the low-resolution video benchmarks and RIVER backbone used in the main experiments, we further evaluate on ACWM-Phys~\citep{xue2026acwmphys}, a recent action-conditioned physical world-model benchmark. We use ACWM-DiT-S, a 205M-parameter flow-matching DiT with a frozen Wan2.1 causal video VAE, operating on $240{\times}240$ videos.

ACWM-Phys covers four distinct physical regimes: rigid-body dynamics, deformable objects, particle-based dynamics, and kinematic systems. We select one task from each regime and evaluate both the in-distribution (IND) and out-of-distribution (OOD) splits. Specifically, we consider \emph{push\_cube}, \emph{push\_rope}, \emph{pour\_water}, and \emph{robot\_arm}. We compare the standard Gaussian source against \method{} under the same flow-matching backbone and evaluation protocol. Results are averaged over three independent seeds and reported using MSE, masked MSE (M-MSE), PSNR, and SSIM.

\begin{table*}[t]
\centering
\caption{\textbf{Results on the ACWM-Phys physical world-model benchmark
at $240{\times}240$ resolution.}
We evaluate one task from each physical regime on both
in-distribution (IND) and out-of-distribution (OOD) splits.
All results are averaged over three independent seeds.}
\label{tab:app_acwm_phys}

\footnotesize
\setlength{\tabcolsep}{5pt}
\renewcommand{\arraystretch}{1.05}

\resizebox{\textwidth}{!}{
\begin{tabular}{llll cccc}
\toprule
Regime
& Task
& Split
& Method
& MSE $\downarrow$
& M-MSE $\downarrow$
& PSNR $\uparrow$
& SSIM $\uparrow$ \\
\midrule

\multirow{4}{*}{Rigid}
& \multirow{4}{*}{push\_cube}
& \multirow{2}{*}{IND}
& Vanilla FM
& 0.002886 & 0.017354 & 26.29 & 0.9546 \\
&
&
& \cellcolor{ourshl}\method{} (Ours)
& \cellcolor{ourshl}\textbf{0.002192}
& \cellcolor{ourshl}\textbf{0.013373}
& \cellcolor{ourshl}\textbf{27.80}
& \cellcolor{ourshl}\textbf{0.9647} \\
\cmidrule(lr){3-8}

&
& \multirow{2}{*}{OOD}
& Vanilla FM
& 0.002977 & 0.016039 & 26.78 & 0.9539 \\
&
&
& \cellcolor{ourshl}\method{} (Ours)
& \cellcolor{ourshl}\textbf{0.002338}
& \cellcolor{ourshl}\textbf{0.012326}
& \cellcolor{ourshl}\textbf{28.13}
& \cellcolor{ourshl}\textbf{0.9629} \\

\midrule

\multirow{4}{*}{Deformable}
& \multirow{4}{*}{push\_rope}
& \multirow{2}{*}{IND}
& Vanilla FM
& 0.000213 & 0.002708 & 36.75 & 0.9878 \\
&
&
& \cellcolor{ourshl}\method{} (Ours)
& \cellcolor{ourshl}\textbf{0.000207}
& \cellcolor{ourshl}\textbf{0.002608}
& \cellcolor{ourshl}\textbf{36.88}
& \cellcolor{ourshl}\textbf{0.9889} \\
\cmidrule(lr){3-8}

&
& \multirow{2}{*}{OOD}
& Vanilla FM
& 0.000331 & 0.004163 & 35.34 & 0.9848 \\
&
&
& \cellcolor{ourshl}\method{} (Ours)
& \cellcolor{ourshl}\textbf{0.000326}
& \cellcolor{ourshl}\textbf{0.004054}
& \cellcolor{ourshl}\textbf{35.46}
& \cellcolor{ourshl}\textbf{0.9859} \\

\midrule

\multirow{4}{*}{Particle}
& \multirow{4}{*}{pour\_water}
& \multirow{2}{*}{IND}
& Vanilla FM
& 0.002630 & 0.012214 & 26.03 & 0.9107 \\
&
&
& \cellcolor{ourshl}\method{} (Ours)
& \cellcolor{ourshl}\textbf{0.001742}
& \cellcolor{ourshl}\textbf{0.007419}
& \cellcolor{ourshl}\textbf{27.90}
& \cellcolor{ourshl}\textbf{0.9292} \\
\cmidrule(lr){3-8}

&
& \multirow{2}{*}{OOD}
& Vanilla FM
& 0.003425 & 0.013788 & 25.02 & 0.8748 \\
&
&
& \cellcolor{ourshl}\method{} (Ours)
& \cellcolor{ourshl}\textbf{0.002514}
& \cellcolor{ourshl}\textbf{0.009194}
& \cellcolor{ourshl}\textbf{26.47}
& \cellcolor{ourshl}\textbf{0.8930} \\

\midrule

\multirow{4}{*}{Kinematics}
& \multirow{4}{*}{robot\_arm}
& \multirow{2}{*}{IND}
& Vanilla FM
& 0.001441 & 0.015606 & 28.70 & 0.9689 \\
&
&
& \cellcolor{ourshl}\method{} (Ours)
& \cellcolor{ourshl}\textbf{0.000286}
& \cellcolor{ourshl}\textbf{0.003000}
& \cellcolor{ourshl}\textbf{35.70}
& \cellcolor{ourshl}\textbf{0.9892} \\
\cmidrule(lr){3-8}

&
& \multirow{2}{*}{OOD}
& Vanilla FM
& 0.006552 & 0.050357 & 22.12 & 0.9021 \\
&
&
& \cellcolor{ourshl}\method{} (Ours)
& \cellcolor{ourshl}\textbf{0.003854}
& \cellcolor{ourshl}\textbf{0.036716}
& \cellcolor{ourshl}\textbf{24.39}
& \cellcolor{ourshl}\textbf{0.9511} \\

\bottomrule
\end{tabular}
}
\end{table*}

As shown in Tab.~\ref{tab:app_acwm_phys}, \method{} improves all four reported metrics across every physical regime on both IND and OOD splits. The improvement varies with the underlying dynamics: the gains are modest on \emph{push\_rope}, while substantially larger improvements are observed on \emph{pour\_water} and \emph{robot\_arm}. In particular, on the OOD \emph{robot\_arm} split, \method{} improves PSNR from $22.12$ to $24.39$ and SSIM from $0.9021$ to $0.9511$. These results show that the benefit of condition-dependent source modeling is not specific to the RIVER backbone or $64{\times}64$ video generation, and transfers to a substantially larger flow-matching DiT under distribution shift.

\section{Additional Analysis}
\label{app:additional_analysis}

\begin{table}[t]
\centering
\caption{\textbf{Effect of flow reversal supervision on source learning.}
The Stage-1 FM is frozen for all variants.}
\vspace{3pt}
\label{tab:app_source_learning_controls}

\footnotesize
\setlength{\tabcolsep}{6pt}
\renewcommand{\arraystretch}{1.0}

\begin{tabular}{lccc}
\toprule
\textbf{Source, frozen FM}
& \textbf{PSNR} $\uparrow$
& \textbf{SSIM} $\uparrow$
& \textbf{FVD} $\downarrow$ \\
\midrule
Standard Gaussian
& 30.29
& 0.8591
& 179.97 \\

CNP w/o flow reversal supervision
& 30.46
& 0.8612
& 179.54 \\

CNP w/ flow reversal supervision
& \textbf{30.85}
& \textbf{0.9008}
& \textbf{100.91} \\
\bottomrule
\end{tabular}
\end{table}













\paratitle{Source-learning controls.}
To verify that the gains of \method{} arise from meaningful source learning rather than additional model parameters, we freeze the Stage-1 FM and compare the same CNP architecture with and without flow reversal supervision. As shown in Tab.~\ref{tab:app_source_learning_controls}, without this supervision, CNP performs similarly to the standard Gaussian source, whereas flow reversal supervision substantially reduces FVD. This isolates the contribution of the supervision signal from the additional CNP parameters.

\paratitle{Prediction diversity.}
To verify that conditioning the source distribution does not reduce intra-condition diversity, we measure the average pairwise LPIPS among predictions generated from the same condition on BAIR64. As shown in Tab.~\ref{tab:app_diversity}, \method{} achieves comparable diversity to the standard Gaussian source, indicating that stochastic prediction diversity is preserved.

\begin{table}[t]
\centering
\caption{\textbf{Prediction diversity on BAIR64.}
We report the average pairwise LPIPS among predictions generated from
the same condition.}
\label{tab:app_diversity}
\vspace{3pt}

\footnotesize
\setlength{\tabcolsep}{6pt}
\renewcommand{\arraystretch}{0.95}

\begin{tabular}{lc}
\toprule
Source & Pairwise LPIPS $\uparrow$ \\
\midrule
Standard Gaussian & 0.0666 \\
\method{} & 0.0686 \\
\bottomrule
\end{tabular}
\end{table}

\paratitle{Iterative source and flow refinement.} To investigate whether further alternating source learning and FM refinement continues to improve performance, we repeat Stages 2--3 for multiple rounds. As shown in Tab.~\ref{tab:app_iterative_refinement}, a second round improves FVD on both BAIR64 and KTH, whereas an additional round provides no consistent further gain. We therefore use the single-round pipeline in the main experiments for simplicity and efficiency.

\begin{table}[t]
\centering
\caption{\textbf{Iterative Stage-2/Stage-3 refinement.}
FVD is reported after each round of source learning and FM refinement.}
\label{tab:app_iterative_refinement}
\vspace{3pt}

\footnotesize
\setlength{\tabcolsep}{6pt}
\renewcommand{\arraystretch}{0.95}

\begin{tabular}{lccc}
\toprule
Dataset
& Round 1
& Round 2
& Round 3 \\
\midrule
BAIR64 
& 61.17 & \textbf{60.48} & 67.07 \\
KTH 
& 44.76 & \textbf{35.04} & 36.37 \\
\bottomrule
\end{tabular}
\end{table}

\paratitle{Robustness of source retrieval via flow reversal.}
To verify that flow reversal supervision is insensitive to the numerical solver used for retrieval, we construct CNP supervision using Dopri, Heun-100, and Euler-100 on BAIR64 $1\!\rightarrow\!15$. As shown in Tab.~\ref{tab:app_reverse_solver}, all three solvers yield similar downstream FVD, indicating that the observed improvement is robust to the choice of numerical solver for flow reversal.

\begin{table}[!ht]
\centering
\caption{\textbf{Robustness to the numerical solver used for flow reversal on BAIR64
prediction.}
Values are mean $\pm$ standard deviation over three seeds.}
\label{tab:app_reverse_solver}
\vspace{3pt}

\footnotesize
\setlength{\tabcolsep}{7pt}
\renewcommand{\arraystretch}{0.95}

\begin{tabular}{lc}
\toprule
Flow reversal solver & FVD $\downarrow$ \\
\midrule
Dopri     & $61.73\pm1.47$ \\
Heun-100  & $\mathbf{61.16\pm0.99}$ \\
Euler-100 & $61.17\pm1.56$ \\
\bottomrule
\end{tabular}
\end{table}


\paratitle{Transport geometry.} To directly examine whether the learned source provides better source--target alignment, we evaluate 256 BAIR64 condition--target pairs using Euler-100. As shown in Tab.~\ref{tab:app_transport_geometry}, \method{} reduces source--target distance, path length, cumulative turning, and endpoint deviation by $9.19\%$, $8.94\%$, $8.51\%$, and $9.69\%$, respectively. These results indicate that the learned source is better aligned with the conditional target region, leading to shorter and less curved transport trajectories.

\begin{table}[!ht]
\centering
\caption{\textbf{Transport geometry on BAIR64.}
Both methods are evaluated on the same 256 condition--target pairs
using Euler-100.}
\label{tab:app_transport_geometry}
\vspace{3pt}

\footnotesize
\setlength{\tabcolsep}{4pt}
\renewcommand{\arraystretch}{0.95}

\begin{tabular}{lcccc}
\toprule
Method
& $\|\bm{x}_0-\bm{x}_1\|_2$ $\downarrow$
& Path length $\downarrow$
& Turning (rad) $\downarrow$
& Endpoint dev. $\downarrow$ \\
\midrule
Standard Gaussian
& 19.1715 & 18.6780 & 0.8292 & 2.7362 \\

\rowcolor{ourshl}
\method{} (Ours)
& \textbf{17.4091}
& \textbf{17.0088}
& \textbf{0.7586}
& \textbf{2.4710} \\
\bottomrule
\end{tabular}
\end{table}
\begin{table}[!ht]
\centering
\caption{\textbf{Alternative Stage-3 couplings on BAIR64
prediction.} \emph{Left}: transport geometry on 256 condition--target pairs using Euler-100. \emph{Right}: downstream generation quality. Direct pairing via flow reversal requires the target and is therefore reported only as a training-time geometry reference.}
\label{tab:app_stage3_coupling}
\vspace{3pt}

\footnotesize
\setlength{\tabcolsep}{3.5pt}
\renewcommand{\arraystretch}{0.95}

\begin{minipage}[t]{0.59\linewidth}
\centering
\resizebox{\linewidth}{!}{
\begin{tabular}{lcccc}
\toprule
Coupling
& Distance $\downarrow$
& Path $\downarrow$
& Turning $\downarrow$
& Endpoint dev. $\downarrow$ \\
\midrule
Flow reversal pair
& \textbf{15.7342}
& \textbf{15.7185}
& \textbf{0.2261}
& \textbf{0.4913} \\

Conditional OT
& 17.3715
& 16.2187
& 1.0610
& 3.9173 \\

\method{} (Ours)
& 17.4091
& 17.0088
& 0.7586
& 2.4710 \\
\bottomrule
\end{tabular}
}
\end{minipage}
\hfill
\begin{minipage}[t]{0.4\linewidth}
\centering
\resizebox{\linewidth}{!}{
\begin{tabular}{lccc}
\toprule
Method
& PSNR $\uparrow$
& SSIM $\uparrow$
& FVD $\downarrow$ \\
\midrule
Conditional OT
& 18.68 & 0.7828 & 101.43 \\

Standard Gaussian
& 20.01 & 0.8121 & 70.16 \\

\method{} (Ours)
& \textbf{20.15}
& \textbf{0.8390}
& \textbf{61.17} \\
\bottomrule
\end{tabular}
}
\end{minipage}
\end{table}

\paratitle{Alternative Stage-3 couplings.}
We further compare alternative couplings for Stage-3 fine-tuning. As shown in Tab.~\ref{tab:app_stage3_coupling}, direct pairing via flow reversal achieves the shortest transport geometry, but requires the target sample and therefore has no inference-time counterpart. Conditional OT also shortens the path relative to \method{}, yet increases turning and endpoint deviation and substantially degrades generation quality. These results indicate that minimizing per-sample transport distance alone is insufficient; the training coupling should remain compatible with the condition-dependent source available at inference.

\section{Temporal Phase 2D Toy Example}
\label{app:2d_toy_example}
\paratitle{Dataset.}
Our 2D toy experiment is designed to faithfully mirror the full three-stage pipeline of conditional generative modeling for temporal sequential data rather than serving as a generic trajectory forecasting benchmark. Each 2D trajectory point is treated as a temporally ordered observation, while temporal conditioning is constructed from the two preceding observations to capture both short-term and longer-range sequential dependencies. The dataset consists of parameterized circular trajectories generated as
\begin{equation}
\bm{x}_1^\alpha = [10 + \cos \alpha, \, 10 + \sin \alpha]^\top,
\end{equation}
where $\alpha$ ranges from $\pi / 180$ to $2\pi$ with $\Delta \alpha = \pi / 360$. For each target point $\bm{x}_1^\alpha$, the condition is defined as $\mathcal{C}_\mathcal{T}^\alpha = [\bm{x}_1^{\alpha-2\Delta \alpha}, \bm{x}_1^{\alpha-\Delta \alpha}]$.

\begin{figure}[t]
    \centering
    \begin{subfigure}{0.58\linewidth}
        \centering
        \includegraphics[width=\linewidth]{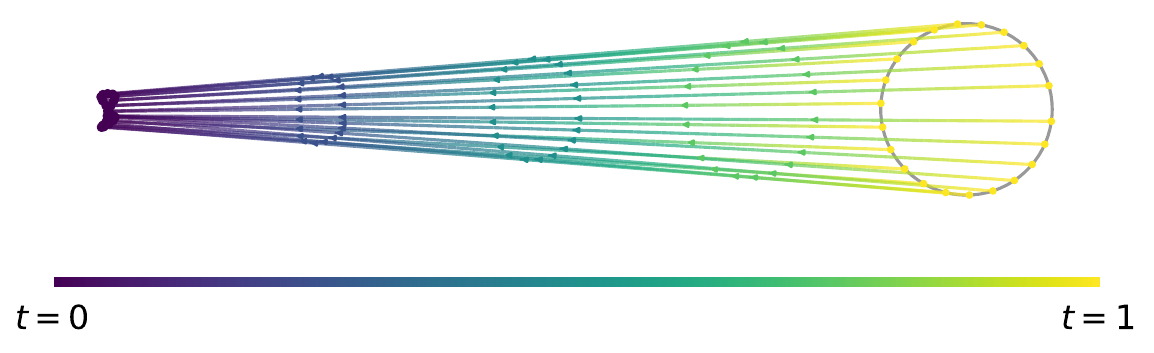}
        \caption{Flow reversal trajectories.}
        \label{fig:reverse_ode_mapping}
    \end{subfigure}
    \hfill
    \begin{subfigure}{0.38\linewidth}
        \centering
        \includegraphics[width=\linewidth]{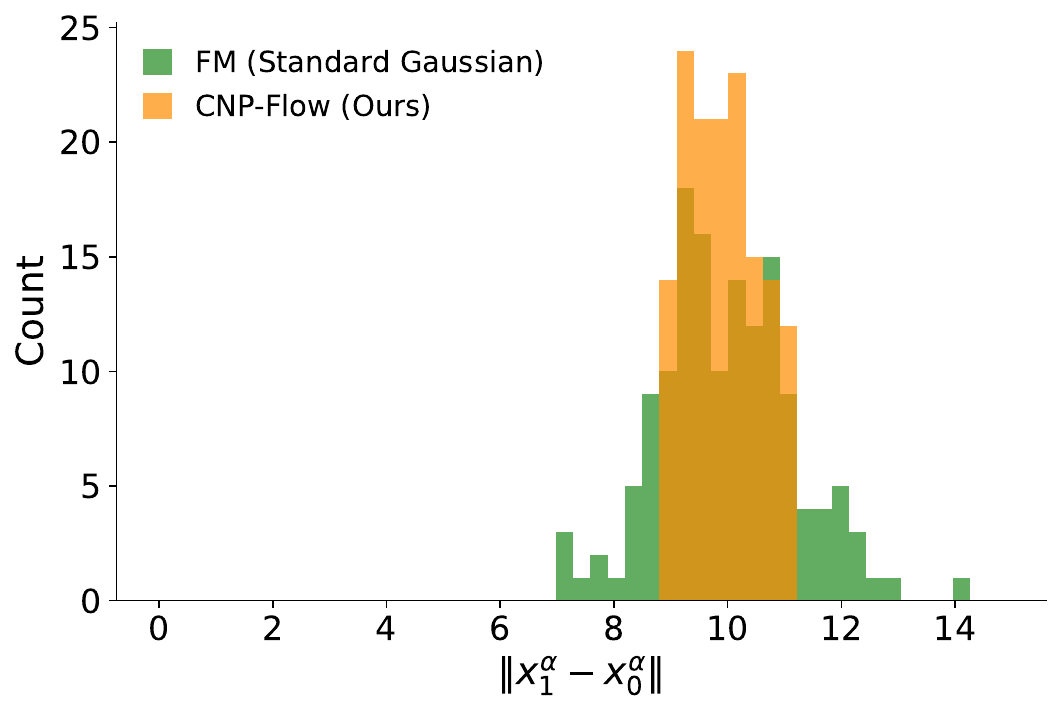}
        \caption{Transport distance.}
        \label{fig:transport_distance}
    \end{subfigure}
    \caption{Toy example visualization of the flow reversal mapping and transport distance analysis.}
    \label{fig:toy_example_analysis}
\end{figure}

\paratitle{Implementation Details.}
All networks are implemented as multi-layer perceptrons (MLPs) with SiLU activation functions. FM model parameterizes the velocity field $\bm{v}_\theta$, taking as input $(\bm{x}_t^\alpha, t, \mathcal{C}_\mathcal{T}^\alpha)$ and outputting $\bm{v}_\theta$, using a 4-layer architecture with hidden dimension 128. The CNP model takes the condition $\mathcal{C}_\mathcal{T}^\alpha$ as input and predicts the parameters of a conditional isotropic Gaussian distribution $(\bm{\mu}_\phi(\mathcal{C}_\mathcal{T}^\alpha), \sigma_\phi(\mathcal{C}_\mathcal{T}^\alpha))$, defining $p_\phi(\bm{x}_0^\alpha\mid\mathcal{C}_\mathcal{T}^\alpha) = \mathcal{N}(\bm{x}_0^\alpha; \bm{\mu}_\phi(\mathcal{C}_\mathcal{T}^\alpha), \sigma_\phi^2(\mathcal{C}_\mathcal{T}^\alpha) \mathbf{I})$, where $\sigma_\phi$ is constrained via a sigmoid activation to lie in $(0,1)$. The CNP consists of 10 layers with a hidden dimension of 128. Our FM shares the same architecture as the baseline but replaces the standard Gaussian source with samples drawn from the learned conditional prior and is initialized from the pretrained baseline FM.

\paratitle{Analysis.}
This behavior can be understood from three observations: (1) As illustrated in Fig.~\ref{fig:toy_example}, vanilla FM transports samples from a standard Gaussian prior to the target distribution. Under neighboring conditions, the corresponding source samples can nevertheless be widely separated, which may lead to intersecting ODE trajectories across timesteps and accumulated deviation from the target distribution in long-horizon generation. (2) Fig.~\ref{fig:toy_example_analysis}(a) shows that flow reversal maps nearby target samples back to nearby source samples, indicating that the pretrained flow preserves local temporal correspondence in the source space. This motivates learning a conditional source prior that preserves proximity among source samples under similar conditions, resulting in more coherent ODE trajectories that better capture the temporal structure of the data. (3) As shown in Fig.~\ref{fig:toy_example_analysis}(b), CNP-Flow produces a more concentrated source distribution with reduced variance compared to FM, which contributes to improved stability and accuracy in temporal generation.

\section{Additional Generated Results}
\label{app:additional_results}
We provide additional qualitative examples for video prediction and interpolation.

\captionsetup{skip=2pt}

\noindent
\begin{minipage}{\linewidth}
    \centering
    \includegraphics[width=0.95\linewidth]{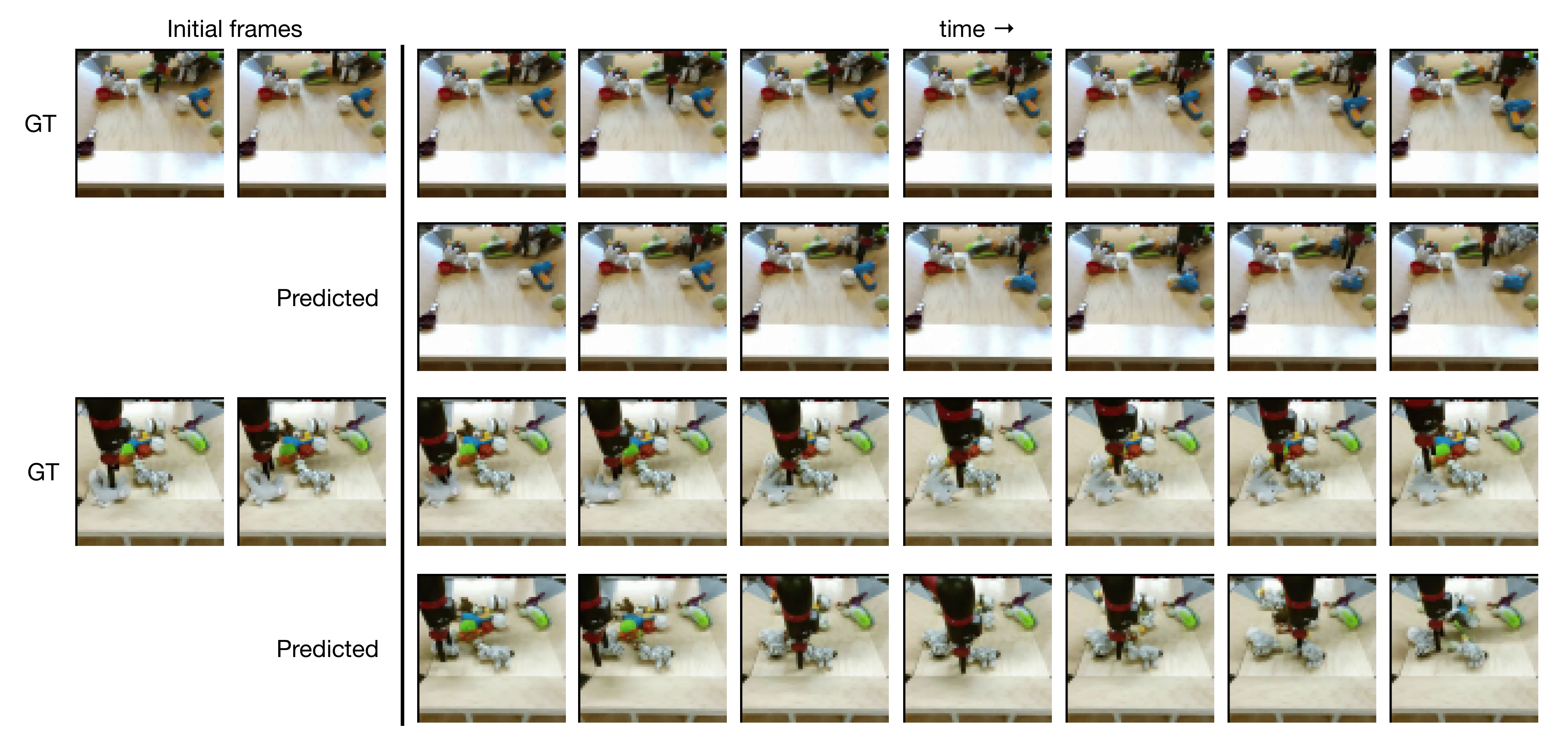}
    \captionof{figure}{\textbf{Additional qualitative results on BAIR video prediction.}}
    \label{fig:app_additional_appendix-1}
\end{minipage}

\vspace{0.5em}

\noindent
\begin{minipage}{\linewidth}
    \centering
    \includegraphics[width=0.95\linewidth]{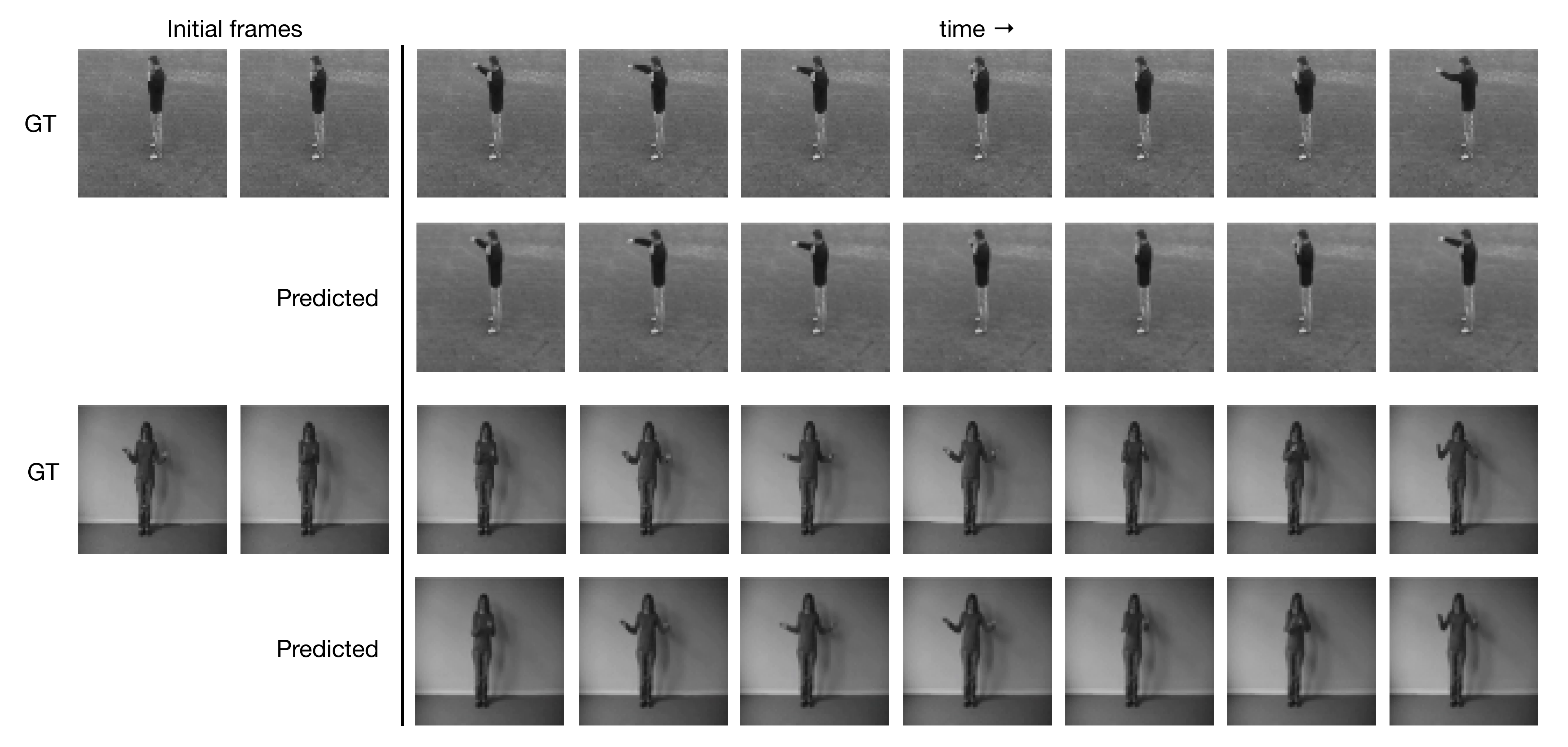}
    \captionof{figure}{\textbf{Additional qualitative results on KTH video prediction.}}
    \label{fig:app_additional_appendix-2}
\end{minipage}

\vspace{0.5em}

\noindent
\begin{minipage}{\linewidth}
    \centering
    \includegraphics[width=0.95\linewidth]{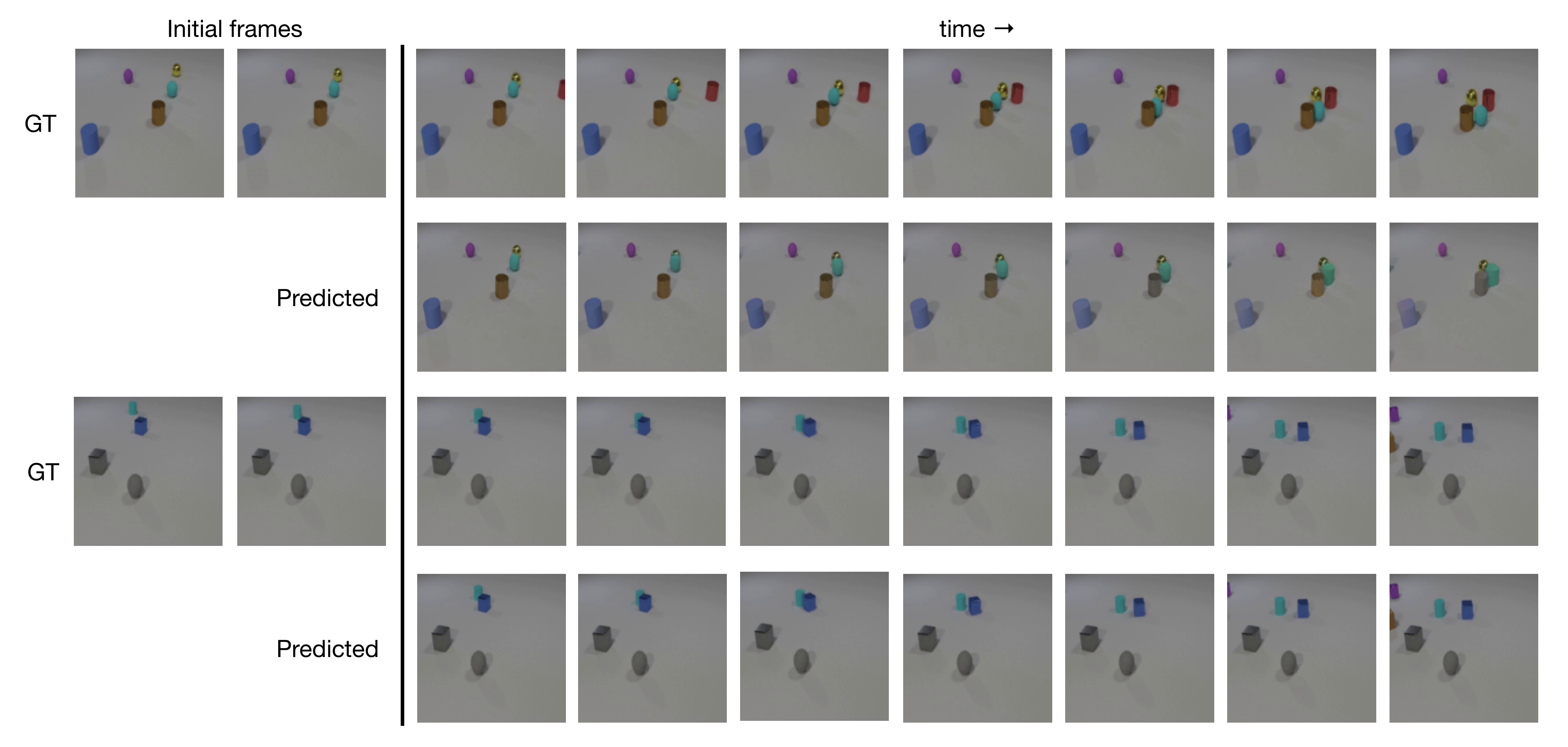}
    \captionof{figure}{\textbf{Additional qualitative results on CLEVRER video prediction.}}
    \label{fig:app_additional_appendix-3}
\end{minipage}

\vspace{0.5em}

\noindent
\begin{minipage}{\linewidth}
    \centering
    \includegraphics[width=0.95\linewidth]{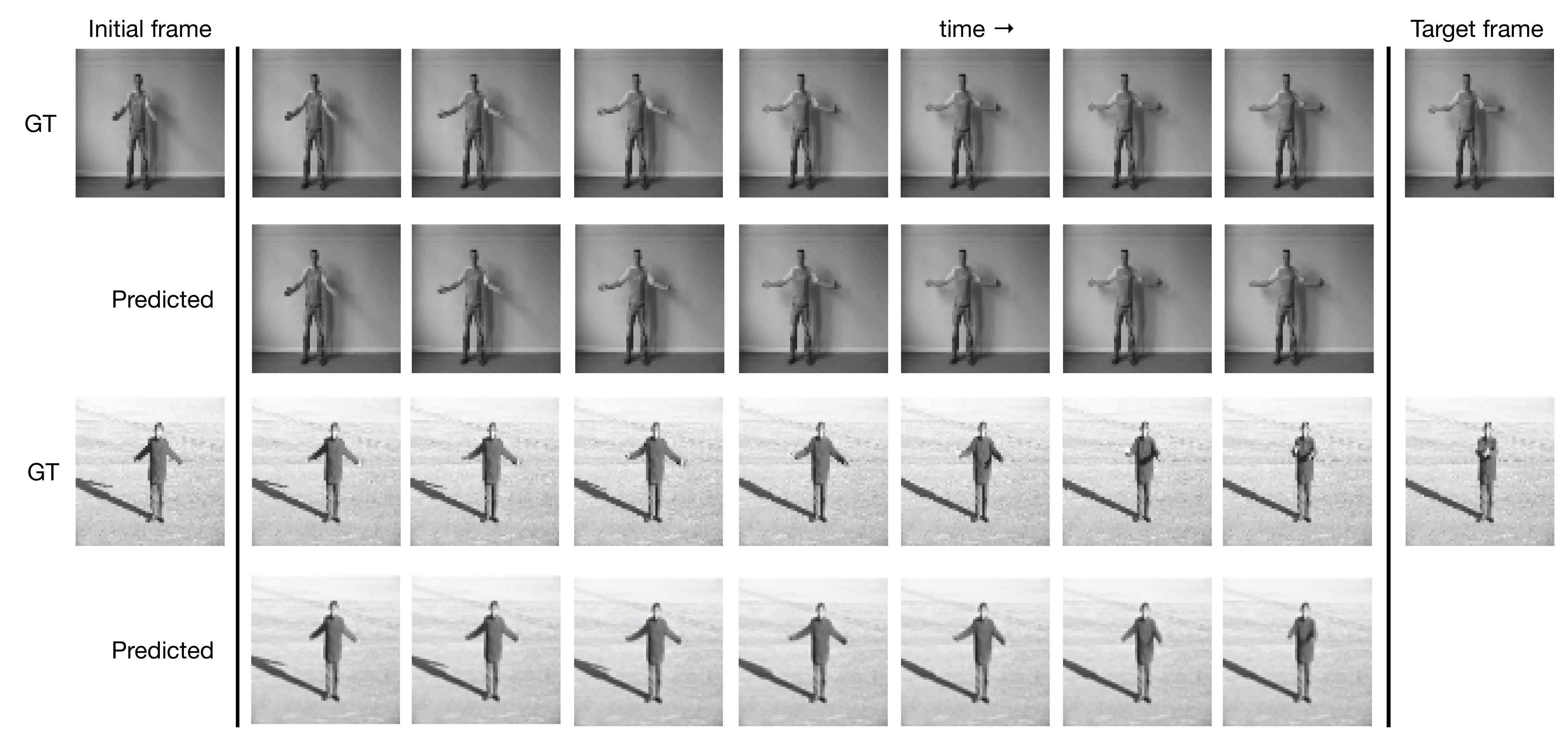}
    \captionof{figure}{\textbf{Additional qualitative results on KTH video interpolation.}}
    \label{fig:app_additional_appendix-4}
\end{minipage}

\vspace{0.5em}

\noindent
\begin{minipage}{\linewidth}
    \centering
    \includegraphics[width=0.95\linewidth]{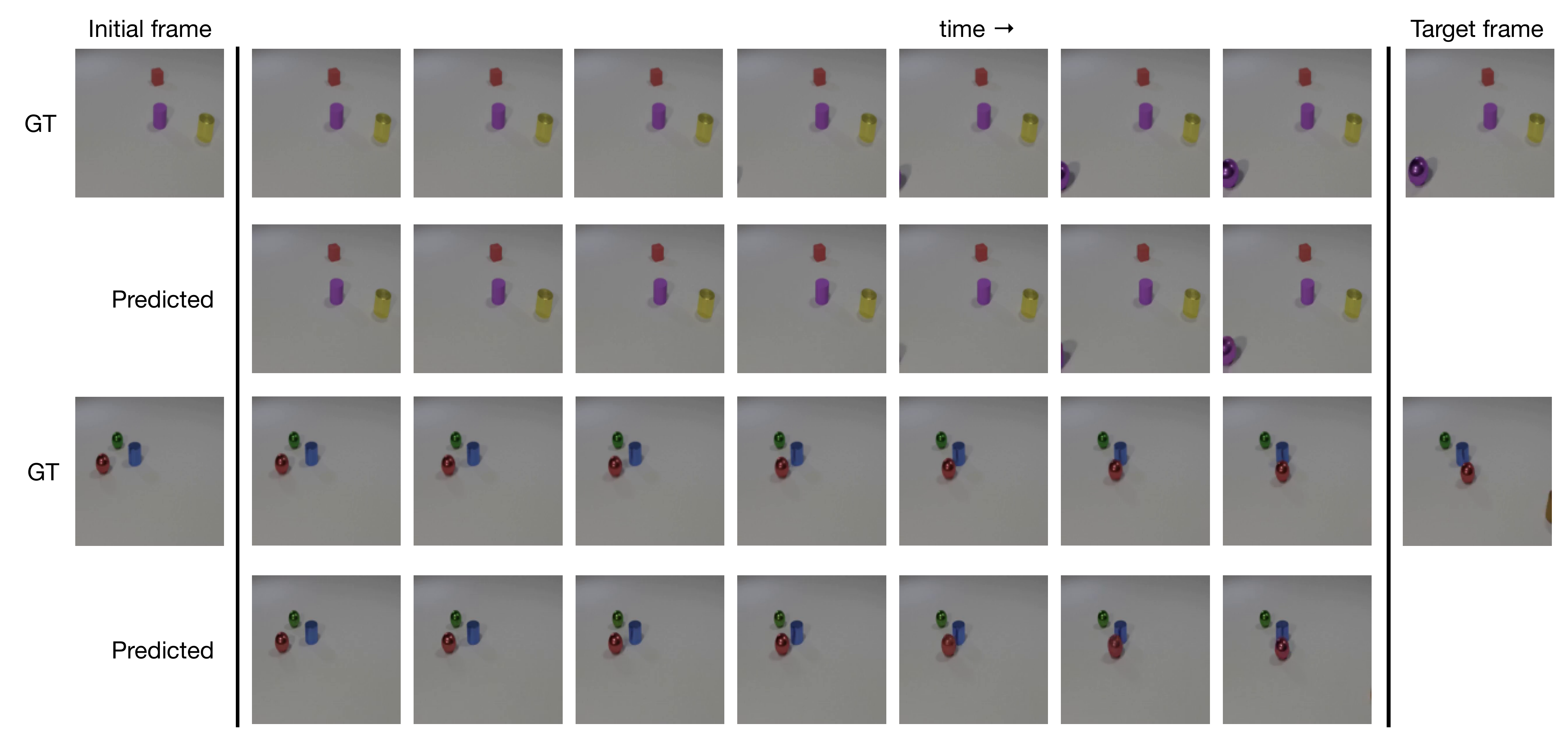}
    \captionof{figure}{\textbf{Additional qualitative results on CLEVRER video interpolation.}}
    \label{fig:app_additional_appendix-5}
\end{minipage}

\end{document}